\documentclass[10pt]{article}
\makeatletter
\renewcommand\@fnsymbol[1]{}
\makeatother
\usepackage{lmodern}
\usepackage[T1]{fontenc}
\usepackage[margin=1.0in]{geometry}
\usepackage{amsmath,amssymb}
\usepackage{bm}
\usepackage{bbm}
\usepackage{booktabs}
\usepackage{hyperref}
\usepackage{pgfplots}
\usepackage{subcaption}
\usepackage{authblk}

\pgfplotsset{compat=1.18}

\usepackage{xcolor}
\usepackage{amsthm}
\usepackage[titletoc,title]{appendix}
\usepackage{caption}
\theoremstyle{plain}
\usepackage{algorithm}
\usepackage{algpseudocode}
\usepackage{utils/commands}
\usepackage{enumitem}

\usepackage[acronym]{glossaries}
\glsdisablehyper
\newacronym{sota}{SOTA}{state-of-the-art}
\newacronym{kpi}{KPI}{key performance indicator}

\newacronym{pac}{PAC}{probably approximately correct}
\newacronym{rv}{RV}{random variable}
\newacronym{iid}{i.i.d.}{independent and identically distributed}

\newacronym{fpr}{FPR}{false positive rate}
\newacronym{tpr}{TPR}{true positive rate}
\newacronym{roc}{ROC}{receiver operating characteristic}
\newacronym{auc}{AUC}{area under curve}
\newacronym{ece}{ECE}{expected calibration error}
\newacronym{kl}{KL}{Kullback-Leibler}
\newacronym{rms}{RMS}{root mean-square}
\newacronym{mav}{MAV}{mean-absolute value}

\newacronym{lti}{LTI}{linear time-invariant}

\newacronym{ai}{AI}{artificial intelligence}
\newacronym{rl}{RL}{reinforcement learning}
\newacronym{es}{ES}{evolution strategies}
\newacronym{rlvr}{RLVR}{reinforcement learning with verifiable rewards}
\newacronym{nn}{NN}{neural network}
\newacronym{rnn}{RNN}{recurrent neural network}
\newacronym{ml}{ML}{maximum likelihood}
\newacronym{map}{MAP}{maximum a posteriori}
\newacronym{ood}{OoD}{out-of-distribution}
\newacronym{gd}{GD}{gradient descent}
\newacronym{erm}{ERM}{empirical risk minimization}
\newacronym{mlp}{MLP}{multi-layer perceptron}
\newacronym{vae}{VAE}{variational auto-encoder}
\newacronym{cdl}{CDL}{contrastive density learning}
\newacronym{elbo}{ELBO}{evidence lower bound}
\newacronym{vi}{VI}{variational inference}
\newacronym{mcmc}{MCMC}{Markov chain Monte Carlo}
\newacronym{em}{EM}{expectation maximization}
\newacronym{vem}{VEM}{variational expectation maximization}
\newacronym{lvm}{LVM}{latent variable model}
\newacronym{hmm}{HMM}{hidden Markov model}
\newacronym{gp}{GP}{Gaussian process}
\newacronym{bnn}{BNN}{Bayesian neural network}
\newacronym{fwp}{FWP}{fast weight programmer}
\newacronym{eb}{EB}{empirical Bayes}
\newacronym{kd}{KD}{knowledge distillation}
\newacronym{spsa}{SPSA}{simultaneous perturbation stochastic approximation}
\newacronym{mezo}{MeZO}{memory-efficient ZO optimizer}
\newacronym{rkhs}{RKHS}{reproducible kernel Hilbert space}
\newacronym{lm}{LM}{language-modeling}

\newacronym{llms}{LLMs}{large language models}
\newacronym{icl}{ICL}{in-context learning}
\newacronym{la}{LA}{linear attention}
\newacronym{ssm}{SSM}{state-space model}
\newacronym{s4}{S4}{structured state-space sequence}
\newacronym{fla}{FLA}{flash linear attention}
\newacronym{gla}{GLA}{gated linear attention}
\newacronym{dyt}{DyT}{dynamic $\mathrm{tanh}$}

\newacronym{ptq}{PTQ}{post-training quantization}
\newacronym{qat}{QAT}{quantization-aware training}
\newacronym{lsq}{LSQ}{learned step-size quantization}
\newacronym{bp}{BP}{backpropagation}

\newacronym{snn}{SNN}{spiking neural network}
\newacronym{ann}{ANN}{artificial neural network}
\newacronym{if}{IF}{integrate-and-fire}
\newacronym{lif}{LIF}{leaky integrate-and-fire}
\newacronym{a2s}{A2S}{ANN-to-SNN}

\newacronym{gpu}{GPU}{graphics processing unit}
\newacronym{lut}{LUT}{look-up table}
\newacronym{acc}{ACC}{accumulation}
\newacronym{bbm}{BBM}{binary-binary multiplication}
\newacronym{qbsm}{QBSM}{quantized-binary static multiplication}
\newacronym{qsm}{QSM}{quantized static multiplication}
\newacronym{qdm}{QDM}{quantized dynamic multiplication}
\newacronym{qnl}{QNL}{quantized non-linearity}
\newacronym{fp}{FP}{floating-point}
\newacronym{fixp}{FixP}{fixed-point}

\newacronym{zo}{ZO}{zeroth-order}
\newacronym{fo}{FO}{first-order}
\newacronym{peft}{PEFT}{parameter-efficient fine-tuning}
\newacronym{sft}{SFT}{supervised fine-tuning}
\newacronym{ste}{STE}{straight-through estimator}
\usetikzlibrary{fillbetween}
\usepgfplotslibrary{fillbetween}
\usepackage{booktabs}
\usepackage{makecell}
\usepackage{pifont}
\usepackage{multirow}
\usepackage{natbib}
\newcommand{\cmark}{\ding{51}}
\newcommand{\xmark}{\ding{55}}

\title{TerMeZO: Ternary Sparse Zeroth-Order Optimization for Fine-tuning BitNet Models at the Edge}
\author{Houssem Sifaou, Prabodh Katti, Bipin Rajendran, and Osvaldo Simeone\thanks{This project is funded by the Advanced Research + Invention Agency (ARIA).}}
\affil{Institute for Intelligent Networked Systems, Northeastern University London, UK}

\date{}

\begin{document}
\maketitle

\maketitle

\begin{abstract}
Fine-tuning \gls{llms} with first-order optimizers requires a memory several times larger than that required for inference. Memory-efficient zeroth-order optimization (MeZO) sidesteps this cost by estimating gradients from forward passes only. However, for BitNet architectures, a family of LLMs with ternary $\{-1,0,1\}$ weights and 8-bit activations, fine-tuning requires updating full-precision latent weights, and thus the memory footprint of MeZO no longer matches that of inference. A promising solution is to finetune only a subset of the latent weights, but existing sparse zeroth-order (ZO) methods either ignore the ternary structure or require first-order gradient information to build a sparse mask, which is at odds with the purpose of ZO fine-tuning. We propose TerMeZO, a sparse MeZO scheme that exploits the geometry of the ternary quantizer itself to identify the latent weights that are more likely to change values during fine-tuning, at no additional data or memory cost. Our convergence analysis shows that TerMeZO can converge faster than full-parameter MeZO, owing to its optimized reduction of the fine-tuning effective dimension. We run extensive experiments on BitNet models ranging from 1B to 3B parameters, spanning classification, instruction-following, and mathematical reasoning tasks. TerMeZO matches or exceeds the performance of full-parameter MeZO while substantially reducing the fine-tuning memory footprint.
\end{abstract}

\glsresetall 

\section{Introduction}
\label{sec:intro}

Deploying \gls{llms} on resource-constrained and edge hardware has become a practical necessity for a growing range of applications, yet it remains difficult, since the memory, compute, and energy footprint of modern LLMs far exceeds what such platforms can afford~\citep{1rajbhandari2020zero,dettmers2022bit}. Extreme quantization has emerged as the most promising route around this barrier. The BitNet family of models~\citep{wang2023bitnet,ma2024era,ma2025bitnetb1582b4ttechnical} trains LLMs from scratch with ternary $\{-1,0,1\}$ weights and 8-bit activations, replacing most matrix multiplications with integer additions and shrinking the memory footprint enough to make LLM inference viable on commodity and edge devices.

Adapting such a model to a new task or user requires fine-tuning, and here the memory advantages of quantization are lost. In fact, fine-tuning with \gls{fo} optimizers must store gradients, optimizer states, and activations, incurring a footprint several times larger than inference~\citep{1rajbhandari2020zero,dettmers2022bit}. Our work is part of the broader effort to relieve this bottleneck and to bring the cost of fine-tuning quantized LLMs back down toward the cost of inference.

\Gls{zo} optimization~\citep{malladi2023fine} offers a promising direction for this goal, as it estimates gradients from forward passes alone through \gls{spsa}~\citep{spall1992multivariate}. Being \gls{bp}-free, \gls{zo} fine-tuning avoids storing gradients and activations, and for standard LLMs, the \gls{mezo}~\citep{malladi2023fine} fine-tunes at roughly the memory cost of inference. However, this parity does not carry over to BitNet models, which are fine-tuned through full-precision \textit{latent weights} that are quantized on the fly at each forward pass, accumulating updates in full precision. As a result, even MeZO must perturb and update every latent weight at each step, so that its memory footprint no longer matches that of inference for the quantized model. 


As an alternative route to memory savings, \gls{peft}~\citep{hu2022lora,houlsby2019adapter,lester2021prompt}, reduces the number of trainable parameters, while still relying on \gls{bp}. Accordingly, \gls{peft} retains much of the activation and gradient overhead that ZO is designed to avoid. Closer to our setting, several sparse \gls{zo} optimization methods build on MeZO to lighten fine-tuning further, yet none are designed for heavily quantized models like BitNet. FIM-S-MeZO~\citep{guo2024zerothorder} restricts fine-tuning to a subset of parameters selected by \gls{fo} gradient information, a process that itself requires \gls{bp} and thus partially negates the memory benefits of ZO. S-MeZO~\citep{liu2024sparse} selects a sparse working set by weight magnitude, a criterion that ignores the ternary decision boundaries. QZO~\citep{shang2026finetuning} avoids \gls{fo} overhead by freezing the integer weights entirely and tuning only the full-precision quantization scales. QZO is effective when the required knowledge is already present in the pre-trained model, but too restrictive on harder tasks that demand acquiring new knowledge.

In this work, we propose TerMeZO, a sparse MeZO scheme tailored to BitNet models that does not require \gls{fo} gradients. Our key observation is that the ternary quantizer itself provides a natural criterion for identifying the latent weights most likely to benefit from fine-tuning. Weights closest to the quantization boundaries are the most sensitive to perturbations, and therefore the most likely to flip their integer value during \gls{zo} fine-tuning. In contrast, weights far from the quantization boundaries are effectively frozen in ternary space. TerMeZO restricts perturbations and updates to this sparse, boundary-adjacent subset of latent weights, while also updating the remaining full-precision parameters in BitNet models \citep{wang2023bitnet,ma2024era,ma2025bitnetb1582b4ttechnical}, namely the embeddings, normalization layers, and \gls{lm} head. Unlike the selection rule in \citep{guo2024zerothorder}, which requires full-parameter \gls{fo} gradients, TerMe\gls{zo}'s selection rule requires only a linear scan of the latent weights. This scan can be performed offline, before the model is uploaded to the edge device, or on-device, provided the device has enough memory to hold the model together with its latent weights.

We first show empirically that this boundary-proximity criterion selects the latent weights carrying the largest gradient signal, confirming that TerMeZO concentrates fine-tuning where it matters most. Then, we provide a convergence analysis proving that TerMeZO can converge faster than full-parameter MeZO by virtue of its reduced effective dimension (Section~\ref{sec:convergence}). The design and theory are validated through extensive experiments on BitNet models ranging from 1B to 3B parameters, including Microsoft's BitNet-b1.58-2B-4T~\citep{ma2025bitnetb1582b4ttechnical} and the Falcon-Edge family~\citep{tiionebitllms}. 
TerMeZO is seen to match full-parameter MeZO, exceeding it in several cases, while consistently outperforming existing sparse MeZO baselines that do not rely on \gls{fo} gradient information (Section \ref{sec:results}).

The main contributions are as follows:
\begin{itemize}[nosep, leftmargin=*]
\item We propose TerMeZO, a sparse MeZO scheme tailored to BitNet models that enables \gls{zo} fine-tuning without sacrificing the memory benefits of quantization, making on-device adaptation at the edge more practical.
\item We provide a convergence analysis establishing that TerMeZO can converge faster than full-parameter MeZO owing to its reduced effective dimension.
\item We demonstrate across BitNet models from 1B to 3B parameters and a range of downstream tasks that TerMeZO matches or exceeds full-parameter MeZO at a reduced memory footprint. 
\end{itemize}


\section{Background and Related Work}
\label{sec:related_work}

Our work sits at the intersection of three lines of research, namely \gls{peft}, \gls{zo} optimization for memory-efficient fine-tuning, and extreme quantization. We review these areas in the following, with additional details provided in Appendix~\ref{app:related_work}.

\noindent\textbf{Parameter-Efficient Fine-Tuning.}
For modern \gls{llms}, full fine-tuning via \gls{fo} optimizers is increasingly impractical, especially for resource-constrained scenarios, requiring a memory footprint several times larger than that of inference to store optimizer states and activations~\citep{1rajbhandari2020zero,dettmers2022bit}. \gls{peft} methods address this bottleneck by exploiting the observation that fine-tuning updates lie in a low-dimensional subspace of the full-parameter space~\citep{aghajanyan2021intrinsic,lialin2023scaling,han2024peft}. They freeze the pre-trained weights and introduce a small set of trainable parameters, achieving performance close to full-parameter fine-tuning. \gls{peft} methods can be broadly grouped into four families, namely adapter-based methods \citep{houlsby2019adapter,pfeiffer2021adapterfusion,he2022unified, mahabadi2021compacter,zhang2022adamix}, soft-prompt and prefix methods \citep{lester2021prompt,liu2022ptuning,li2021prefix,lester2021prompt,he2022unified}, low-rank reparameterization \citep{hu2022lora,dettmers2023qlora,ICLR2024_e6c2e85d,zhang2023adalora,liu2024dora,kopiczko2024vera}, and sparse fine-tuning \citep{guo2021diffpruning,zaken2022bitfit,sung2021training,ansell2022composable,xu2024random}.\\
\textbf{Zeroth-Order Optimization and MeZO.}
\gls{zo} optimization estimates gradients from function evaluations alone, without requiring a backward pass. A widely adopted \gls{zo} estimator is \gls{spsa}~\citep{spall1992multivariate}, which approximates the gradient $\nabla {F}(\mathbf{w})$ of the objective function $F(\w)$  along a random direction $\mathbf{z}\sim\mathcal{N}(\mathbf{0},\mathbf{I}_d)$ as
\begin{equation}
    \widehat{\nabla {F}}(\mathbf{w}) \;=\; \frac{F(\mathbf{w}+\epsilon\mathbf{z}) - F(\mathbf{w}-\epsilon\mathbf{z})}{2\epsilon}\,\mathbf{z},
    \label{eq:spsa}
\end{equation}
where $\epsilon$ is a small constant, requiring only two function evaluations per step. Other choices of the random direction $\z$ include uniform~\citep{duchi2015optimal} and Rademacher~\citep{dang2026fzoo}. \gls{zo} optimization has historically been applied to non-differentiable or black-box objectives, arising in settings such as adversarial attacks~\citep{chen2017zoo,ilyas2018black} and reinforcement learning~\citep{salimans2017evolution}, where gradients are inherently unavailable. More recently, \textit{memory-efficient ZO} (MeZO)~\citep{malladi2023fine} demonstrated that SPSA can also be used to fine-tune LLMs using only the memory footprint of inference. MeZO’s success in fine-tuning hinges on the locally well-conditioned loss landscape of pre-trained LLMs near initialization, which bypasses the dimension dependence of standard \gls{zo} convergence bounds~\citep{malladi2023fine}. A growing body of follow-up work has sought to improve MeZO along several directions. In particular, reduced-variance methods~\citep{zhao2025hizoo,cai2026adamezo,dang2026fzoo,dbouk2026adaptivity} has been investigated in to enhance convergence. Other schemes optimize the perturbation using curvature-aware~\citep{Seung2025Loren} or activation-guided~\citep{lin2026agzo} criteria, and gradient-informed perturbations \citep{mi2025towards}. These methods are compatible with TerMeZO and they could be combined with it, although we do not explore this direction in this work. Information about curvature can be also leveraged to dynamically select which weights to update \citep{wang2026curvzo}. However, such approach incur a memory requirement higher than full-parameter \gls{mezo}, and we therefore consider them poorly suited to the near-inference-memory of quantized models regime we target in this work.



\noindent\textbf{Fine-Tuning Quantized Models.}
Quantization replaces full-precision weights and activations with low-bit representations to reduce memory and compute. Post-training quantization (PTQ) methods such as GPTQ~\citep{frantar2022gptq}, AWQ~\citep{lin2024awq}, and SmoothQuant~\citep{xiao2023smoothquant} achieve 4- and 8-bit inference at a minor accuracy gap, while quantization-aware training (QAT) recovers the gap by simulating quantization during training~\citep{nagel2021white}. \textit{Extreme} quantization, 1-bit and ternary weights, has been studied for energy-constrained deployment. BinaryConnect~\citep{courbariaux2015binaryconnect}, XNOR-Net~\citep{rastegari2016xnor}, and ternary weight networks~\citep{li2016ternary} demonstrated that aggressive weight quantization is feasible for vision models. More recently, the BitNet framework~\citep{wang2023bitnet,ma2024era} has introduced training LLMs from scratch with $\{-1,0,1\}$ weights and 8-bit activations, enabling matrix-multiplication-free inference kernels~\citep{ma2024era,ma2025bitnetb1582b4ttechnical}. To further increase the availability of BitNet models, recent works have also explored pipelines for distilling existing pretrained models into BitNet versions~\citep{wu2025bitnet}. BitNet distillation is orthogonal to our work, as it is typically applied prior to fine-tuning.

The \gls{ste}~\citep{bengio2013estimating} is the standard \gls{bp} surrogate in this regime, treating the non-differentiable quantizer as the identity in the backward pass. However, STE is known to be a biased gradient surrogate~\citep{yin2019understanding} and to interact poorly with optimizer states stored at low precision~\citep{li2021brecq}. This further motivates the use of \gls{zo} for fine-tuning extremely quantized models.

FIM-S-MeZO~\citep{guo2024zerothorder} combines \gls{zo} fine-tuning with sparsity, fine-tuning a tiny fraction of parameters identified at pre-training, while keeping the remaining weights quantized to 4 bits. The work~\citep{yang2024adazeta} reduces the effective \gls{zo} dimension by inserting adapters and adapting the per-step query budget to stabilize convergence. QuZO~\citep{zhou2025quzo} adopts a stochastic rounding scheme that lowers gradient bias for quantized weights, allowing fine-tuning with low precision forward passes. However, we found that this approach fails for ternary quantization (See Appendix \ref{app:quzo}). Most directly related to our work, QZO \citep{shang2026finetuning} fine-tunes quantized models with Me\gls{zo} by keeping the quantized weights fixed and fine-tuning only the quantization scales as full-precision parameters. We instead target complex tasks that may require fine-tuning of the ternary weights.

Table~\ref{tab:comparison} summarizes the related work for memory-constrained, BitNet-style deployment along five axes. \gls{fo} \gls{peft} methods \citep{hu2022lora,dettmers2023qlora} require storing full-precision adapters, gradients, and activations throughout training; evaluate matrix multiplication in full-precision; and do not benefit from sparse updates. Among \gls{zo} methods, sparse Me\gls{zo} variants \citep{liu2024sparse,guo2024zerothorder}  and TerMe\gls{zo} save memory by restricting the perturbation and update to a sparse full-precision working set. TerMe\gls{zo} stands out by obtaining this set directly from the ternary quantization, thus avoiding both the \gls{fo} optimization steps required to construct a mask in~\citep{guo2024zerothorder} and the scale-only update constraints of QZO~\citep{shang2026finetuning}.

\begin{table}[ht]
\centering
\caption{Comparison of fine-tuning methods for BitNet-style models (Matmul = matrix multiplication).}
\label{tab:comparison}
\small
\setlength{\tabcolsep}{3pt}
\begin{tabular}{lccccc}
\toprule
\textbf{Method}
  & \makecell{Near inference \\ memory of BitNet}
  & \makecell{Matmul-free\\(BitNet like)}
  & \makecell{No \gls{fo}\\overhead}
  & \makecell{Sparsity} 
  &  \makecell{Quantization-\\Aware Mask}\\
\midrule
\multicolumn{5}{l}{\textit{\gls{peft} --- \gls{fo} fine-tuning}} \\[2pt]
LoRA~\citep{hu2022lora}          & \xmark & \xmark & \xmark & \xmark & \xmark \\
QLoRA~\citep{dettmers2023qlora}  & \xmark & \xmark & \xmark & \xmark & \xmark\\
\midrule
\multicolumn{5}{l}{\textit{\gls{zo} fine-tuning}} \\[2pt]
MeZO~\citep{malladi2023fine}              & \xmark & \cmark & \cmark & \xmark & \xmark\\
Me\gls{zo} (LoRA)~\citep{malladi2023fine}       & \xmark & \xmark & \cmark & \xmark & \xmark\\
S-MeZO~\citep{liu2024sparse}         & \cmark & \cmark & \cmark & \cmark & \xmark\\
FIM-S-MeZO~\citep{guo2024zerothorder}    & \xmark & \cmark & \xmark & \cmark & \xmark\\
LoZO~\citep{chen2025enhancingzerothorderfinetuninglarge} & \xmark & \xmark & \cmark & \xmark& \xmark \\
\midrule
\multicolumn{5}{l}{\textit{\gls{zo} fine-tuning --- quantized models}} \\[2pt]
QZO~\citep{shang2026finetuning} & \cmark & \cmark & \cmark & \xmark  & \xmark\\
QuZO~\citep{zhou2025quzo}                              & \cmark & \cmark & \cmark & \xmark & \xmark\\
\midrule
\textbf{TerMe\gls{zo} (ours)} & \cmark & \cmark & \cmark & \cmark & \cmark\\
\bottomrule
\end{tabular}
\end{table}

\section{TerMeZO}
\label{sec:method}
In this section, we propose a sparse version of Me\gls{zo} tailored to BitNet models, a family of LLMs with ternary weights \citep{ma2024era, ma2025bitnetb1582b4ttechnical}. BitNet models apply ternary quantization to most linear projections of the transformer blocks, while keeping the token embeddings, normalization modules, and the language modeling (LM) head in full precision~\citep{ma2025bitnetb1582b4ttechnical,wang2023bitnet}.
 Accordingly, ternarization is applied to every linear projection in the transformer blocks, i.e., the query, key, value, and output projections of the multi-head attention modules, as well as to all linear layers in the multi-layer perceptron (MLP) blocks. In the sequel, we refer to these layers as \textit{ternary linear projections}. The proposed TerMe\gls{zo} preserves the ternary arithmetic of BitNet in the ternary linear projections, perturbing a subset of the latent weights, which are quantized on-the-fly during the forward pass.
 
 
 \noindent\textbf{Ternary Quantization.} Let $\w \in \mathbb{R}^d$ denote the vector collecting all latent weights of the ternary linear projections, and $\btheta \in \mathbb{R}^p$ the vector collecting all full-precision parameters of the remaining blocks. The ternary weights dominate the parameter count, i.e., we have $p \ll d$. For example, for the Bitnet model {Falcon-E-3B-Instruct}~\citep{tiionebitllms}, we have the ratio $p /d=0.044$. During full-parameter fine-tuning, each ternary linear projection maintains full-precision \textit{latent weights} that are quantized on-the-fly in the forward pass. The latent weights $\w$ act as an accumulator for gradient updates. To elaborate, denote as $\w^\ell \in \mathbb{R}^{n^\ell}$ the vectorized latent weights of a ternary linear projection indexed by integer $\ell$, where $n^\ell$ is the number of parameters in that projection. The ternary weight $\widetilde{\w}^\ell$ is obtained element-wise as
\begin{align}
 [\widetilde{\w}^\ell]_{i} &=Q([{\w}^\ell]_{i})=  \begin{cases}
      1 \quad \ \ \text{if } [\w^\ell]_{i} > \tau^\ell\\
      -1 \quad \text{if } [\w^\ell]_{i} < -\tau^\ell\\
      0 \quad  \text{  otherwise},
  \end{cases}
  \label{eq:quant}
\end{align}
with threshold $\tau^\ell= 0.5\,(\gamma^\ell + \delta)$, where $\gamma^\ell = \sum_{i=1}^{n^\ell}|w_{i}^\ell|/{n^\ell}$ is the average absolute value of the weights, and $\delta>0$ is a small constant for numerical stability (e.g., $\delta=10^{-5}$). 
 Accordingly, weights far from either threshold are effectively frozen in ternary space when subject to small updates, while weights in the immediate vicinity of the thresholds are the most sensitive to adaptation. Based on this observation, TerMe\gls{zo} restricts fine-tuning to the subset of latent weights closest to the decision boundaries.

\noindent\textbf{TerMeZO.}
In TerMeZO, a subset of latent weights $\w^\ell$ and the corresponding quantization threshold $\tau^\ell$ are updated at each fine-tuning iteration $t$. We denote by $\w^\ell_t$ the current weight vector of ternary linear projection $\ell$ at iteration $t$ and by $\tau_t^\ell$ the corresponding threshold. For a given latent-weight vector $\w^\ell_t$, the distance of the $i$-th entry $[{\w}^\ell_t]_{i}$ to the quantization boundary is measured by
\begin{equation}
    d_{i,t}^\ell(\w^\ell_t) = \min\bigl(|[\w^\ell_t]_i - \tau_t^\ell|,\; |[\w^\ell_t]_i + \tau_t^\ell|\bigr),
    \label{eq:distance}
\end{equation}
where $[\cdot]_i$ denotes the $i$-th element of the argument vector. TerMeZO takes as input a target fraction $\rho_0 \in (0,1]$ of latent weights $[{\w}^\ell]_{i}$ to be fine-tuned. To this end, based on the pretrained model, it selects the subsets
\begin{equation}
    \mathcal{S}_0^\ell = \Bigl\{ i : d_{i,0}^\ell\bigl(\w_0^\ell\bigr) \leq \xi_0 \Bigr\},
    \label{eq:selection_init}
\end{equation}
where the threshold $\xi_0$ is selected so that a fraction $\rho_0$ of the total number of latent weights, across all layers, is retained. Concretely, this is done by computing the distances of all latent weights (across all projection layers) to their decision boundaries,  setting $\xi_0$ to the $\rho_0$-quantile of all distances. A validation of this criterion based on a comparison with \gls{bp} can be found in Appendix \ref{app:val_criterion}.

The sets $\{\mathcal{S}_0^\ell\}_{\ell=1}^L$ are computed once, offline, and the complexity of this step is linear in the number of latent weights. For the following iterations $t \geq 1$, we allow the active set only to shrink by retaining only previously-active indices whose latent weights still lie within distance $\xi_0$ of the decision boundaries as
\begin{equation}
    \mathcal{S}_t^\ell = \Bigl\{ i \in \mathcal{S}_{t-1}^\ell : d_{i,t}^\ell\bigl(\w_t^\ell\bigr) \leq \xi_0 \Bigr\}.
    \label{eq:selection_}
\end{equation}
We denote the sparsification mask by $\mathbf{m}^\ell_t \in \{0,1\}^{n^\ell}$, with $[\mathbf{m}^\ell_t]_i = 1 \iff i \in \mathcal{S}_t^\ell$. We further let $\mathbf{m}_t = \{\mathbf{m}^\ell_t\}_{\ell=1}^L$ denote the vector collecting the masks across all layers, and we write the masked latent weights as $\overline{\mathbf{w}}_t = \mathbf{m}_t \odot \mathbf{w}_t$. By construction of set \eqref{eq:selection_}, we have the inclusion $\mathcal{S}_t^\ell \subseteq \mathcal{S}_{t-1}^\ell$ via \eqref{eq:selection_}, so that the fraction of active latent weights is non-increasing in the iteration index $t$ and upper-bounded by the prescribed budget $\rho_0$. Once a weight drifts outside the band $\xi_0$ around its threshold, it is permanently frozen for the remainder of fine-tuning.

{The pseudo-code of the TerMe\gls{zo} algorithm can be found in Appendix~\ref{app:pseudocode}.} At each iteration $t$, TerMe\gls{zo} updates the sparse sets $\mathcal{S}_{t}^\ell$ via \eqref{eq:selection_init}, and performs a standard \gls{zo} update on the full-precision parameters $\btheta_t$ and the selected latent weights $\overline{\w}_t $. To elaborate, we focus here on \gls{sft}, and refer to Appendix \ref{app:rl_app} for the application of TerMeZO for \gls{rl}-based fine-tuning as a form of~\gls{es}~\citep{qiu2025evolution}.

Let \( f(\tilde{\mathbf{w}}, \boldsymbol{\theta}; \mathbf{x}, y) \) represent the sample-wise supervised loss for input-output data point $(\mathbf{x}, y)$, which is parameterized by ternary weights \( \tilde{\mathbf{w}} \) and full-precision parameters \( \boldsymbol{\theta} \). Accordingly, the population loss for BitNet models can be expressed as a function of the full-precision weights $\w$ and $\btheta$ as
\begin{align}
    \tilde F(\w,\btheta) = F(Q(\w),\btheta) = \mathbb{E}_{(\mathbf{x},y)}\bigl[f(Q(\w),\btheta;(\mathbf{x},y))\bigr],
    \label{eq:objective}
\end{align}
with $Q(\cdot)$ being the quantization function and $\mathbb{E}_{(\mathbf{x},y)}[\cdot]$ representing the expectation over the data distribution. A batch loss estimate using a mini-batch $\mathcal{B}$ of size $|\mathcal{B}|$ is denoted as
\begin{align}
    \widehat{\widetilde{F}}(\w,\btheta; \mathcal{B}) = \frac{1}{|\mathcal{B}|}\sum_{(\mathbf{x},y) \in\mathcal{B}}\bigl[f(Q(\w),\btheta;(\mathbf{x},y))\bigr].
    \label{eq:objective_est}
\end{align}
At each iteration $t$, $K$ Gaussian perturbations are applied to both the masked latent weights, i.e., the elements $i$ of vector  $\overline{\w}_t = \m_t\odot\w_t$ with $[\m_t]_i=1$, and to the full-precision parameters $\btheta_t$. A two-sided SPSA estimator is used, where the loss is evaluated along each perturbation direction with scale $\epsilon>0$ on a shared mini-batch $\mathcal{B}_t$, and the resulting finite differences are averaged over the $K$ perturbations to update the model parameters. 





\section{Convergence Analysis}
\label{sec:convergence}
In this section, we study the convergence of TerMeZO. To facilitate the analysis, we assume that a single perturbation is applied at each iteration, i.e., $K=1$, but the extension to general $K$ follows by standard averaging arguments~\citep{malladi2023fine}). The objective function $\tilde F(\w,\btheta)$ in \eqref{eq:objective} is non-differentiable with respect to latent weights $\w$, making it difficult to analyze the convergence to a local optimizer. However, \gls{spsa} with Gaussian perturbations $\z\sim \mathcal{N}(\mathbf{0}, \mathbf{I})$ provides an unbiased estimator of the gradient of the Gaussian smoothed-objective~\citep{Nesterov2015RandomGM}
\begin{equation}
\tilde{F}_\epsilon(\w, \btheta) = \mathbb{E}_{\z\sim \mathcal{N}(\mathbf{0}, \mathbf{I})}[F(\w+\epsilon \z, \btheta)],
    \label{eq:smoothed_obj}
\end{equation}
i.e.,
\begin{align}
    \mathbb{E}_{\z\sim \mathcal{N}(\mathbf{0}, \mathbf{I})}\!\left[
        \frac{\tilde F(\w+\epsilon \z,\btheta) - \tilde F(\w-\epsilon \z,\btheta)}{2\epsilon}\,\z
    \right]
    \;=\; \nabla_\w \tilde F_\epsilon(\w,\btheta).
\end{align}
For completeness, a proof of this result can be found in Appendix \ref{app:spsa}. As in~\citep{ghadimi2013stochastic,liu2018zeroth,malladi2023fine}, we study convergence to a stationary point with respect to the smoothed objective \eqref{eq:smoothed_obj} through the time-averaged expected squared gradient norm
\begin{align}
\frac{1}{T}\sum_{t=0}^{T-1}\mathbb{E}\|\nabla \tilde F_\epsilon(\w_t,
\btheta_t)\|^2=\frac{1}{T}\sum_{t=0}^{T-1}
    \mathbb{E}\bigl[\|\nabla_\w \tilde F_\epsilon(\w_t,\btheta_t)\|^2
    + \|\nabla_\btheta \tilde F_\epsilon(\w_t,\btheta_t)\|^2\bigr].
    \label{eq:avd_grad_norm}
\end{align}

To this end, we make the following four assumptions, which are formally stated in Appendix~\ref{app:assumptions}:
\textit{Assumption 1)} The objective function $\tilde F(\w,\btheta)$ is uniformly bounded;
\textit{Assumption 2)} The loss function $F(\w,\btheta)$ is $L$-Lipschitz in $\w$  and has $L$-Lipschitz gradients in $\btheta$;
 \textit{Assumption 3)} The masked gradient estimator captures a $c$-fraction of the gradient energy of the smoothed objective $\tilde F_\epsilon(\w,\btheta)$ along the latent weights $\w$; and \textit{Assumption 4)} the minibatch gradients have variance at most $\sigma_B^2$. Furthermore, we define as
$k_t$ the number of
active coordinates during iteration $t$, and introduce the \textit{average boundary-crossing fraction} $q_t$ to measure the average fraction of active weights that change value after a perturbation:
\begin{align}
    k_t = \sum_{\ell=1}^L\sum_{i=1}^{n^\ell}[\m_t^\ell]_i, \quad \quad q_t = \frac{1}{k_t}\sum_{\ell=1}^L \sum_{i:\,[\m_t^\ell]_i=1} 
    P \!\bigl[Q([\w_t^\ell + \epsilon \z^\ell]_i) 
    \neq Q([\w_t^\ell]_i)\bigr]. 
   \label{eq:avg-flip}
\end{align}
We donote by
$    N_t = \sum_{\ell=1}^L \sum_{i=1}^{n^\ell}
    \E \!\bigl[\bigl(Q([\w_t^\ell + \epsilon \z^\ell]_i)
    - Q([\w_t^\ell]_i)\bigr)^2\bigr]
$
 the expected number of latent weights that change value after a perturbation.


\begin{theorem}
\label{thm:main}
Let Assumptions~\ref{ass:bounded}--\ref{ass:variance} hold. Define
$L_\w = {\sqrt{2k_0}\,L}/{\epsilon^2}$
   and 
    $L_{\max}=\max\bigl\{L_\w,\;8c^2L(p+2)\bigr\}$. Consider TerMe\gls{zo} with constant step sizes $\eta = c/(2 L_\w)$ for
the latent weights $\w$ and $\mu = c^2/L_{\max}$ for the
full-precision parameters $\btheta$. Then, after $T$ steps of TerMeZO algorithm, the time-averaged gradient norm \eqref{eq:avd_grad_norm} satisfies the inequality
\begin{align}
    \frac1T\sum_{t=0}^{T-1}\E\|\nabla\tilde F_\epsilon(\w_t,\btheta_t)\|^2
    \;\leq\;
    &\frac{2L_{\max}\Delta_0}{c^2T}
   +\mathcal{O}\!\left( \epsilon^2 p^2 + \overline{N} + \sigma_B^2
      + L_{\max}\,\overline{\sigma^2_\w} \right),
    \label{eq:main_bound}
\end{align}
where $\overline{N} = \frac{1}{T}\sum_{t=0}^{T-1}N_t$; $\overline{\sigma^2_\w} = \frac{1}{T}\sum_{t=0}^{T-1}\E[\sigma^2_{\w,t}]$, with $\sigma^2_{\w,t}=k_t^2q_t+\sqrt3\,k_t\sqrt{q_t}$; $\Delta_0 = \tilde F_\epsilon(\w_0,\btheta_0) - \tilde F_\epsilon^*$ is the initial optimality gap; $c\in(0,1]$ is the mask-overlap constant of Assumption~\ref{ass:sparse}; and $\sigma_B^2$ is the minibatch variance of Assumption~\ref{ass:variance}. 
\end{theorem}

We make the following remarks about connections of Theorem~\ref{thm:main} with prior art:
\begin{itemize}[nosep, leftmargin=*]
    \item When no quantization is applied, only the terms corresponding to the full-precision parameters $\btheta$ remain in the bound \eqref{eq:main_bound}, and thus standard Me\gls{zo} convergence rate $O(p/T)$ for $p$ trainable parameters is recovered~\citep{malladi2023fine,liu2018zeroth}.
    \item When the initial mask size $k_0$ dominates the number of trainable parameters, i.e., when $k_0\gg p$, we have $L_{\max}=L_\w=\sqrt{2k_0}\,L/\epsilon^2$, and the bound \eqref{eq:main_bound} gives an overall convergence rate $O\!\bigl(\sqrt{k_0}/(\epsilon^2 T)\bigr)$ to a neighborhood of a stationary point, which is governed by the number $k_0$ of active latent weights.
    \item A convergence rate of $O(k/T)$ was proved for sparse MeZO~\citep{guo2024zerothorder}, where $k$ is the mask size. Setting $\epsilon^2 = 1/\sqrt{k_0}$ in \eqref{eq:main_bound} recovers this rate. Note that this scale gives the range $\epsilon \in [ 3\cdot 10^{-3}, 10^{-2}]$ for active-parameter counts $k_0 \in [10^{8},\, 10^{10}]$, matching the typical order of magnitude of the perturbation scale $\epsilon$ used in
    practice~\citep{malladi2023fine}. 
\end{itemize}

\section{Memory-Footprint Analysis}
\label{sec:mem_analysis}
In this section, we elaborate on the memory savings afforded by TerMeZO over \gls{bp} and full-parameter MeZO. We start by relating the memory footprints of full-parameter MeZO and TerMeZO for the same Bitnet model, and we then provide an experimental evaluation to support the analysis. The key difference between TerMeZO and full-parameter MeZO is that the former keeps a full-precision copy only of the active latent weights, while also additionally storing the sparsification mask. Let $B$ be the number of bytes per full-precision weight, e.g., $B=2$ for bfloat16, and let $B_Q$ be the effective number of bytes per packed ternary weight, e.g., $B_Q=2/8$ for BitNet models \citep{wang2023bitnet,ma2025bitnetb1582b4ttechnical}. Downgrading a frozen latent weight from full precision to ternary saves $B-B_Q$ bytes. A fraction $1-\rho_0$ of the weights are frozen in TerMeZO, while each of the $\rho_0 d$ active weights additionally retains its $B_Q$-byte quantized copy. Hence, the peak TerMeZO memory footprint $M_{\text{TerMeZO}}$ can be expressed as a function of the memory footpring of MeZO $M_{\text{MeZO}}$ as
\begin{equation}
    M_{\text{TerMeZO}}
    \;=\;
    M_{\text{MeZO}}
    \;-\;\lceil d\,(1-\rho_0)\,(B-B_Q)\rceil
    \;+\;\lceil \rho_0 d\, B_Q\rceil
    \;+\;M_{\text{mask}}.
    \label{eq:memory}
\end{equation}
The mask size in \eqref{eq:memory} is given by $M_{\text{mask}}=\min\bigl(\lceil d/8 \rceil,\ \lceil \rho_0 d\lceil\log_2 d\rceil/8\rceil\bigr)$
bytes, where the first term corresponds to the option of a dense bit-mask storing
one bit per latent weight, and the second to an active-index list that records the
$\lceil\log_2 d\rceil$-bit index of each active weight.


Table~\ref{tab:memory_footprint} reports the measured memory footprint on GPU of BP, full-parameter MeZO, and TerMeZO, along with that of pure inference as a lower bound, for fine-tuning of Falcon-E-1B and Falcon-E-3B on an H100 GPU. Measuring the footprint of TerMeZO and of inference each requires a custom kernel. A BitNet inference kernel for GPUs is already available~\citep{ma2025bitnetb1582b4ttechnical}, and we implement an analogous kernel for TerMeZO. The table also shows the estimate \eqref{eq:memory} for comparison. The results in Table~\ref{tab:memory_footprint} confirm the significant memory savings of MeZO relative to BP \citep{katti2025device}. TerMeZO improves further on MeZO by a wide margin. In particular, at sparsity level $\rho_0=0.05$, it reduces the footprint by roughly $3\times$ relative to MeZO, and by around $4\times$ at level $\rho_0=0.01$, bringing the memory cost of fine-tuning close to that of inference.

\begin{table}[ht]
\centering
\caption{Memory footprint (GB) during fine-tuning of Falcon-E-1B and Falcon-E-3B on an H100 GPU. For TerMeZO, we report both measured and estimated values using \eqref{eq:memory}. The TerMeZO and inference numbers assume $2$-bit storage of the ternary weights. Results here are measured usi \gls{sft} with the context length
fixed to 512. }
\label{tab:memory_footprint}
\small
\begin{tabular}{lccccc}
    \toprule
                        & \textbf{BP} & \textbf{MeZO} & \textbf{TerMeZO } & \textbf{TerMeZO} & \textbf{Inference}\\
                        &              &               &    ($\rho_0=0.05$) & ($\rho_0=0.01$) &
                        \\
    \midrule
    \textbf{Falcon-E-1B} & 16.08 & 3.61 & \textbf{1.33} (est. 1.31) & \textbf{0.94} (est. 0.93) & 0.89 \\
    \textbf{Falcon-E-3B} & 28.50 & 6.28 & \textbf{2.08} (est. 2.04) & \textbf{1.36} (est. 1.34) & 1.24 \\
    \bottomrule
\end{tabular}
\end{table}

\section{Empirical Results}
\label{sec:results}

We conduct experiments on a set of publicly available BitNet models across a range of tasks. Specifically, we evaluate on Microsoft's BitNet-b1.58-2B-4T~\citep{ma2025bitnetb1582b4ttechnical}, a 2B-parameter model pre-trained on 4T tokens, and the Falcon-Edge family~\citep{tiionebitllms}, a set of BitNet models ranging from 1B to 3B parameters designed for on-device deployment. Through these experiments, we aim to answer the following questions:
1) Does TerMe\gls{zo} match the performance of full-parameter Me\gls{zo} on BitNet models? 2) How does TerMe\gls{zo} compare to existing sparse and \gls{peft} variants of MeZO on BitNet models?

\noindent\textbf{GLUE/SuperGLUE Classification Tasks.} Here, we benchmark TerMe\gls{zo} against the following baselines: (i) standard \gls{bp} fine-tuning with Adam optimizer; (ii) in-context learning (ICL) \citep{brown2020language}; (iii) full-parameter MeZO~\citep{malladi2023fine}, which applies \gls{zo} updates to all latent weights and full-precision parameters; (iv) Me\gls{zo} with LoRA (LoRA-MeZO); (v) FIM-S-MeZO~\citep{guo2024zerothorder}; and (vi) QZO~\citep{shang2026finetuning}, a Me\gls{zo} variant tailored to quantized models that perturbs the quantization scales rather than the latent weights themselves. Following the conventional Me\gls{zo} evaluation protocol~\citep{malladi2023fine}, fine-tuning is cast as language modeling over a task-specific prompt template. We report results on six standard classification benchmarks from the GLUE~\citep{wang2018glue} and SuperGLUE~\citep{wang2021adversarial} suites; namely {SST2}~\citep{socher2013recursivesst2}, {RTE}~\citep{dagan2005pascal}, {CoPA}~\citep{roemmele2011choice},  {CB} (CommitmentBank)~\citep{de2019commitmentbank}, {BoolQ}~\citep{clark2019boolq}, and {MultiRC}~\citep{MultiRC2018}. For all tasks, we report the accuracy on the validation set in Table~\ref{tab:glue_tasks} after fine-tuning using the BitNet-b1.58-2B-4T model. All \gls{zo} methods are run for the same number of optimization steps to ensure a fair comparison. Hyperparameters and experiment details can be found in Appendix \ref{app:glue_exp}.

We observe that TerMeZO is as effective as full-parameter MeZO fine-tuning on most of these tasks, and even exceeds it on some tasks. QZO, which fine-tunes only the scales of the ternary projections, is also observed to perform well. This is expected as these tasks are relatively simple, implying that the model already possesses the required knowledge and tuning only a subset of the full-precision parameters is sufficient. In contrast, LoRA-MeZO lags behind the other baselines. We attribute this to the known limitation of \gls{zo} optimization as a method for training from scratch \citep{malladi2023fine}. In fact, unlike full fine-tuning, which starts from a well-conditioned pretrained initialization, LoRA introduces randomly initialized low-rank matrices. This effect is compounded in our setting, since, in order to remain compatible with the BitNet architecture, we constrain the LoRA projections to be ternary.

\begin{table}[h]
\centering
\caption{Accuracy $(\%)$ of TerMeZO and baselines on classification tasks using the {Microsoft/bitnet-b1.58-2B-4T} BitNet model~\citep{ma2025bitnetb1582b4ttechnical}. The fraction of trainable latent weights is fixed to $\rho_0 = 0.05$ for TerMeZO.}
\label{tab:glue_tasks}
\small
\begin{tabular}{lcccccc}
    \toprule
                            & SST2      & RTE   & CoPA  & CB    & BoolQ     & MultiRC   \\
    \midrule
    \textbf{BP}             & 94.8      & 84.1  & 83    & 88.0  & 86.1      & 83.4      \\
    \textbf{ICL}            & 85.4      & 75.4  & 80    & 78.5  & 76.8      & 73.1      \\
    \midrule
    \textbf{MeZO}           & 93.2      & 81.9  & 78    & 87.9  & 82.3      & 82.2      \\
    \textbf{LoRA-MeZO}      & 55.0      & 71.6  & 73    & 55.5  & 54.5      & 66.6      \\ 
    \textbf{FIM-S-MeZO}     &  92.3       & 82.4  & 75    & 83.8   & 79.5   & 81.6       \\
    \textbf{QZO}            & 91.7      & 79.7  & 77    & 82.8  & 82.9      & 81.8      \\
    \textbf{TerMeZO (Ours)} & 93.1      & 83.4  & 80    & 84.8  & 82.7      & 83.4      \\
    \bottomrule
\end{tabular}
\end{table}

\noindent\textbf{Instruction Following and Reasoning Tasks.} We now consider more challenging fine-tuning tasks that potentially require the acquisition of new knowledge, and for which optimizing the ternary weights is required. Specifically, we adopt an instruction-following dataset, MagiCoder~\citep{wei2023magicoder}, which measures the ability of the fine-tuned model to follow natural language instructions and generate correct code, as well as a mathematical reasoning task, GSM8K~\citep{cobbe2021gsm8k}. We include here the S-MeZO baseline with two selection rules, namely maximum weight magnitude and minimum weight magnitude \citep{liu2024sparse}.


\begin{table}[h]
\centering
\caption{Performance of TerMeZO with $\rho_0=0.05$ and benchmarks on GSM8K and MagiCoder using two BitNet models, Falcon-E-1B-Base and Falcon-E-3B-Base. We report accuracy (\%) on GSM8K and pass@1 (\%) on MagiCoder. }
\label{tab:gsm8k_magicoder}
\small
\begin{tabular}{lcccc}
    \toprule
    & \multicolumn{2}{c}{\textbf{Falcon-E-1B-Base}} & \multicolumn{2}{c}{\textbf{Falcon-E-3B-Base}} \\
    \cmidrule(lr){2-3} \cmidrule(lr){4-5}
                            & GSM8K     & MagiCoder     & GSM8K     & MagiCoder     \\
    \midrule
    \textbf{Base Model}     & 14.0      & 17.1          & 20.0     & 37.8           \\
    \textbf{BP}             & 37.7      & 32.9          & 57.4     & 49.4           \\
    \midrule
    \textbf{MeZO}           & 32.6      & 25.6          & 45.3     & 42.7           \\
 \textbf{FIM-S-MeZO}        & 32.5      & 25.0          & 44.5     & 43.3           \\
  \textbf{S-MeZO (min)}     & 27.8      & 15.8          & 26.8     & 41.5           \\
   \textbf{S-MeZO (max)}    & 22.1      & 14.6          & 27.4     & 42.1           \\
    \textbf{QZO}            & 18.6     & 16.4          & 23.3     & 42.0           \\
    \textbf{TerMeZO (Ours)} & 31.8     & 24.4          & 44.4     & 43.3           \\
    \bottomrule
\end{tabular}
\end{table}

Table~\ref{tab:gsm8k_magicoder} reports results on GSM8K and MagiCoder for two BitNet models, Falcon-E-1B-Base and Falcon-E-3B-Base. For GSM8K, we report the test accuracy,
whereas for the models fine-tuned on MagiCoder, we report pass@1 on HumanEval~\citep{chen2021evaluating}, the percentage of problems for which the single sampled code completion passes all unit tests.  TerMeZO substantially improves over the base model and closely tracks full-parameter MeZO. QZO, by contrast, fails to match TerMeZO, indicating that for these tasks, fine-tuning the quantization scales alone is insufficient. We also observe a consistent performance gap between BP and TerMeZO for a given model size. However, TerMeZO's lower memory footprint makes it possible to fine-tune larger models than BP would allow under the same memory budget. For example, the 3B model fine-tuned with TerMeZO outperforms the 1B model fine-tuned with BP, despite the latter requiring considerably more memory on a smaller model size as shown in Table~\ref{tab:memory_footprint}. 
S-MeZO, whose mask selection relies on latent weight magnitude and ignores the quantizer geometry, falls behind by a significant margin. In contrast, while FIM-S-MeZO matches the performance of TerMeZO, it requires initial SGD steps on a separate dataset prior to fine-tuning, resulting in a memory footprint comparable to \gls{bp}. This performance parity with FIM-S-MeZO further validates the effectiveness of TerMeZO's selection criteria.

Although FIM-S-MeZO performs similarly to TerMeZO in these experirmnts, it does so under stronger requirements (e.g., access to \gls{fo} gradient information). We also observe that TerMeZO is more robust to the choice of learning rate. Specifically, for the same fraction of latent weights, some learning rates (e.g., $5\times10^{-4}$) cause FIM-S-MeZO to diverge, whereas TerMeZO remains stable (see Appendix~\ref{lrs_ablation}). We attribute this to the fact that a large proportion of the weights selected by FIM-S-MeZO are not reachable by perturbations, since they lie far from the quantization boundaries, which yields pure noise in the updates. Finally, we note that when an \gls{fo} warmup before fine-tuning is affordable (e.g., offline in the cloud), combining the selection criteria of TerMeZO and FIM-S-MeZO could capture the benefits of both. We leave this investigation to future work.

\section{Conclusion}
\label{sec:ccl}
\noindent\textbf{Summary.} This paper proposed TerMeZO, a sparse MeZO scheme for LLM fine-tuning tailored to BitNet models. The sparse mask selection exploits the ternary weight geometry and is inexpensive to evaluate for this class of models, requiring only a pass through all latent weights before fine-tuning. A convergence analysis of the proposed scheme is provided. Extensive experiments on a variety of datasets confirm the effectiveness of TerMeZO, showing that it recovers, and in some cases exceeds, the performance of full-parameter MeZO, while enabling a substantial reduction in the memory required for fine-tuning. By reducing the fraction of latent weights needed during fine-tuning, fine-tuning BitNet models can be carried out on edge devices with a near-inference memory footprint.

\noindent\textbf{Limitations.} An efficient implementation of TerMeZO requires custom, hardware-dependent kernels to realize the memory savings in practice. Future work could target such kernels with a
focus on both memory and latency. Additionally, due to the limited availability of larger open-source BitNet models, we have only evaluated our scheme on models ranging from 1B to 3B parameters. Validating the approach at larger scales is left for future work.

\noindent\textbf{Perspectives.} A natural extension of this work is to evaluate TerMeZO on BitNet-equipped linear-attention or hybrid-attention models, which would allow us to jointly address the memory footprint of the KV cache alongside that of the latent weights. This is particularly relevant for \gls{rl}-based fine-tuning where long generation rollouts can make KV cache storage a significant bottleneck. Another important direction is to validate the memory savings reported in this paper on real edge hardware with custom kernels, beyond the GPU-based setup used in our experiments.

\bibliographystyle{iclr2027_conference}
\bibliography{ref}

\newpage
\begin{appendices}

\section{Extended Related Work}
\label{app:related_work}

\subsection{PEFT methods}
Full fine-tuning of modern LLMs with first-order optimizers is often prohibitively expensive, particularly in resource-constrained settings, due to the substantial memory required for optimizer states and activations~\citep{1rajbhandari2020zero,dettmers2022bit}. Parameter-efficient fine-tuning (PEFT) methods mitigate this cost by freezing the pre-trained model and optimizing only a small number of additional parameters, motivated by the observation that fine-tuning updates typically lie in a low-dimensional subspace of the full parameter space~\citep{aghajanyan2021intrinsic,lialin2023scaling,han2024peft}. PEFT methods can be broadly classified into four families:

\noindent\textbf{Adapter-based methods.}
Adapters insert small trainable modules between frozen layers of the backbone~\citep{houlsby2019adapter,pfeiffer2021adapterfusion}. Variants reduce the additional inference cost by parallelizing adapters with the residual stream~\citep{he2022unified}, sharing them across layers~\citep{mahabadi2021compacter}, or making them sparsely activated~\citep{zhang2022adamix}.

\noindent\textbf{Soft-prompt and prefix methods.}
Prompt-tuning prepends a small set of trainable continuous tokens to the input embeddings~\citep{lester2021prompt,liu2022ptuning}, while prefix-tuning extends this to every transformer layer by injecting trainable key-value pairs into the attention mechanism~\citep{li2021prefix}. These methods are extremely parameter-efficient but are known to be sensitive to initialization and to underperform on smaller backbones~\citep{lester2021prompt,he2022unified}.

\noindent\textbf{Low-rank reparameterization.}
LoRA~\citep{hu2022lora} freezes the pre-trained weights and learns a low-rank additive update, so that only a small number of parameters are trained while inference cost remains unchanged after merging. It is typically applied to attention and feed-forward layers and has become the standard \gls{peft} method. Extensions include QLoRA~\citep{dettmers2023qlora}, QA-LoRA~\citep{ICLR2024_e6c2e85d}, AdaLoRA~\citep{zhang2023adalora}, DoRA~\citep{liu2024dora}, and VeRA~\citep{kopiczko2024vera}.

\noindent\textbf{Sparse parameter fine-tuning.}
A complementary line of work updates only a sparse subset of the original model weights, motivated by the lottery-ticket hypothesis~\citep{frankle2018lottery} and by the observation that fine-tuning induces sparse changes~\citep{guo2021diffpruning}. BitFit updates only bias terms~\citep{zaken2022bitfit}; Diff-Pruning learns a sparse difference vector with an $L_0$ relaxation~\citep{guo2021diffpruning}; FishMask selects parameters by Fisher information~\citep{sung2021training}; and SAM/PaFi-style methods select coordinates by saliency or magnitude~\citep{ansell2022composable,xu2024random}. Some recent works have also explored sparsity for \gls{zo} fine-tuning of \gls{llms}. Existing sparse \gls{zo} fine-tuning methods select weights using criteria that either ignore the quantization structure or recover it only at high cost. S-MeZO~\citep{liu2024sparse} picks weights by latent-weight magnitude, which is agnostic to the ternary decision boundaries, while FIM-S-MeZO~\citep{guo2024zerothorder} relies on first-order gradient information collected during pre-training to build its sparse mask. In practice, such information is often unavailable, necessitating first-order optimization on the target fine-tuning task or on a separate dataset as a prerequisite before the \gls{zo} fine-tuning process can even begin.

\subsection{Zeroth-Order Optimization and MeZO}
\gls{zo} optimization has traditionally targeted non-differentiable or black-box
objectives, as encountered in adversarial attacks~\citep{chen2017zoo,ilyas2018black}
and reinforcement learning~\citep{salimans2017evolution}, where analytical gradients
are unavailable by construction. MeZO~\citep{malladi2023fine} showed that the same
\gls{spsa} estimator can fine-tune LLMs at the memory cost of inference alone, and a
substantial body of subsequent work has sought to improve upon it. Sparse MeZO
(S-MeZO)~\citep{liu2024sparse} restricts perturbations to a subset of weights selected
by magnitude. MeZO-SVRG~\citep{gautam2024variancereducedzerothordermethodsfinetuning} reduces variance via control variates. HiZOO incorporates diagonal Hessian estimates~\citep{zhao2024helene}. Lo\gls{zo} and its variants combine Me\gls{zo} with low-rank perturbations~\citep{chen2025enhancingzerothorderfinetuninglarge}. Block-coordinate variants update only a subset of layers per step~\citep{zhang2024revisiting}. Several recent works~\citep{guo2024zerothorder,zhou2025quzo,shang2026finetuning} extend Me\gls{zo} to quantized models, which is the regime closest to ours and which we discuss next. More recently, \gls{zo} optimization has also proven effective for RL-based fine-tuning of \gls{llms}~\citep{qiu2025evolution,xu2026quantized}. 

\subsection{ZO Fine-tuning for Quantized Models}
Severals efforts have explored fine-tuning quantized models using \gls{zo} optimization without needing \gls{ste}. QuZO~\citep{zhou2025quzo} performs \gls{zo} fine-tuning entirely through low-precision forward passes, replacing the STE with optimized stochastic rounding to reduce gradient bias in settings with INT4 and INT8 weights. More recently, quantized evolution strategies~(QES)~\citep{xu2026quantized} fine-tune quantized LLMs directly in the discrete parameter space via accumulated error feedback. QZO~\citep{shang2026finetuning} proceeds by freezing the quantized weights entirely and fine-tuning the full-precision quantization scales solely. We show in our experiments that this is not sufficient for some fine-tuning tasks.

\newpage
\section{Empirical Evaluation of the Selection Criterion}
\label{app:val_criterion}

 To further motivate the selection criterion in \eqref{eq:selection_}, we now evaluate experimentally the relationship between to the gradient of the population loss \eqref{eq:objective} over the latent weights and the distance to quantization boundaries~\eqref{eq:distance}. We aim to verify whether during a run of \gls{bp} fine-tuning via stochastic gradient descent (SGD), implemented via \gls{ste}~\citep{bengio2013estimating}, there is sufficient overlap between the weights with the largest gradient signal and the weights close to the decision boundary as per the selection criterion \eqref{eq:selection_}. 

To this end, we partition the latent weights of each ternary linear projection layer $\ell$ at SGD step $t$ into bins $\{B_k\}$ according to their normalized distance to the nearest decision boundary, $\tilde{d}_{i,t}^\ell = d_{i,t}^\ell(\w^\ell_t)/\tau_t^\ell$ with distance $d_{i,t}^\ell(\w^\ell_t)$ defined in \eqref{eq:distance}, and we let $\mathcal{I}_k^{\ell,t} = \{i : \tilde{d}_{i,t}^\ell \in B_k\}$ be the corresponding index set. The per-coordinate gradient magnitude, averaged over all weights within each bin, across all ternary linear projection layers $\ell$, and across $T$ SGD steps, is
\begin{equation}
    G_k =  \frac{1}{TL}\sum_{t=1}^{T} \sum_{\ell =1}^L \frac{1}{|\mathcal{I}_k^{\ell,t}|}\sum_{i\in \mathcal{I}_k^{\ell,t}}\left|\frac{\partial {\tilde F}(\w_t,\btheta_t)}{\partial [\w_{t}^\ell]_i}\right|,
    \label{eq:avg_grad_mag}
\end{equation}
with $ {|\mathcal{I}_k^{\ell,t}|}$ denoting the cardinality of $\mathcal{I}_k^{\ell,t}$ and $\partial \tilde{F}(\w_t, \btheta_t)/\partial [\w_t^\ell]_i$ denoting the partial derivative of function $\tilde{F}$ with respect to parameter $[\w_t^\ell]_i$. Figure~\ref{sub1} plots the per-coordinate gradient magnitude \eqref{eq:avg_grad_mag} as a function of the normalized distance to quantization boundaries $\tilde{d}_{i,t}^\ell$ for the SST2 \citep{socher2013recursivesst2} and MultiRC \citep{MultiRC2018} datasets. The results show that the gradient magnitudes are substantially higher for latent weights lying close to the decision boundaries, providing empirical support for the boundary-proximity selection criterion in \eqref{eq:selection_}. Moreover, Figure~\ref{sub2} shows the ratio of the sparsified gradient norm $\|\m_t\odot \nabla_{\w} \tilde{F}(\w,\btheta)\| $ to the full gradient norm $\|\nabla_{\w_t} \tilde{F}(\w_t,\btheta_t)\|$ during $100$ SGD steps using the two datasets. It is observed that, even a small portion of the latent weights, namely $\rho_t=0.01$, recovers a significant fraction of the full gradient norm.
 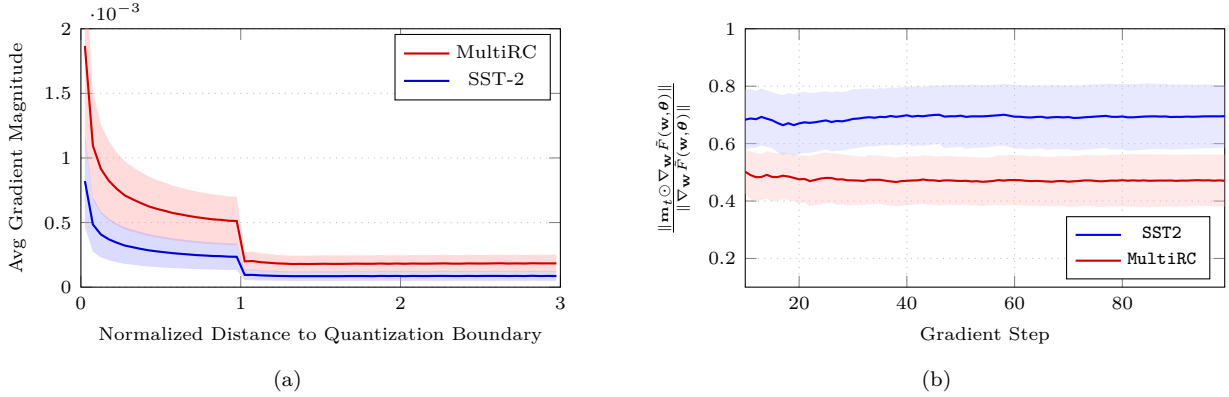
\begin{figure}[ht]
  \centering
  \begin{subfigure}[t]{0.48\linewidth}
    \centering
     \begin{tikzpicture}
    \begin{axis}[
      width=\linewidth, height=5cm,
      xlabel={Normalized Distance to Quantization Boundary},
      ylabel={Avg Gradient Magnitude},
      xmin=0.0, xmax=3, ymin=0, ymax=0.002,
    unbounded coords=discard,
    restrict x to domain=0:3,
      xtick={0,1,2,3},
      yticklabel style={/pgf/number format/fixed, /pgf/number format/precision=1},
      ymajorgrids=true,
      grid style={dotted,gray!50},
      tick label style={font=\scriptsize},
      label style={font=\scriptsize},
      legend style={font=\scriptsize, at={(0.98,0.98)}, anchor=north east},
    ]

      \addplot[name path=upper1, draw=none, forget plot] coordinates {
    (0.0250,0.002604) (0.0750,0.001511) (0.1250,0.001260) (0.1750,0.001131) (0.2250,0.001042) (0.2750,0.000973) (0.3250,0.000930) (0.3750,0.000890) (0.4250,0.000856) (0.4750,0.000829) (0.5250,0.000809) (0.5750,0.000786) (0.6250,0.000771) (0.6750,0.000754) (0.7250,0.000740) (0.7750,0.000729) (0.8250,0.000718) (0.8750,0.000712) (0.9250,0.000702) (0.9750,0.000701) (1.0250,0.000272) (1.0750,0.000275) (1.1250,0.000262) (1.1750,0.000258) (1.2250,0.000250) (1.2750,0.000247) (1.3250,0.000244) (1.3750,0.000244) (1.4250,0.000244) (1.4750,0.000245) (1.5250,0.000247) (1.5750,0.000245) (1.6250,0.000245) (1.6750,0.000246) (1.7250,0.000247) (1.7750,0.000246) (1.8250,0.000246) (1.8750,0.000249) (1.9250,0.000246) (1.9750,0.000248) (2.0250,0.000246) (2.0750,0.000249) (2.1250,0.000248) (2.1750,0.000248) (2.2250,0.000249) (2.2750,0.000250) (2.3250,0.000248) (2.3750,0.000248) (2.4250,0.000250) (2.4750,0.000249) (2.5250,0.000249) (2.5750,0.000251) (2.6250,0.000251) (2.6750,0.000249) (2.7250,0.000250) (2.7750,0.000249) (2.8250,0.000251) (2.8750,0.000251) (2.9250,0.000251) (2.9750,0.000251) 
          };
      \addplot[name path=lower1, draw=none, forget plot] coordinates {
      (0.0250,0.001129) (0.0750,0.000678) (0.1250,0.000570) (0.1750,0.000515) (0.2250,0.000478) (0.2750,0.000447) (0.3250,0.000428) (0.3750,0.000411) (0.4250,0.000395) (0.4750,0.000383) (0.5250,0.000373) (0.5750,0.000363) (0.6250,0.000355) (0.6750,0.000348) (0.7250,0.000342) (0.7750,0.000337) (0.8250,0.000332) (0.8750,0.000329) (0.9250,0.000325) (0.9750,0.000324) (1.0250,0.000128) (1.0750,0.000129) (1.1250,0.000123) (1.1750,0.000121) (1.2250,0.000117) (1.2750,0.000116) (1.3250,0.000115) (1.3750,0.000115) (1.4250,0.000115) (1.4750,0.000116) (1.5250,0.000116) (1.5750,0.000115) (1.6250,0.000115) (1.6750,0.000116) (1.7250,0.000116) (1.7750,0.000116) (1.8250,0.000115) (1.8750,0.000117) (1.9250,0.000116) (1.9750,0.000117) (2.0250,0.000116) (2.0750,0.000117) (2.1250,0.000116) (2.1750,0.000116) (2.2250,0.000117) (2.2750,0.000118) (2.3250,0.000117) (2.3750,0.000117) (2.4250,0.000118) (2.4750,0.000117) (2.5250,0.000117) (2.5750,0.000118) (2.6250,0.000118) (2.6750,0.000117) (2.7250,0.000117) (2.7750,0.000117) (2.8250,0.000118) (2.8750,0.000118) (2.9250,0.000118) (2.9750,0.000118) 
      };
      \addplot[red!40, fill opacity=0.35, forget plot] fill between[of=upper1 and lower1];
      \addplot[red!80!black, thick] coordinates {
        (0.0250,0.001867) (0.0750,0.001095) (0.1250,0.000915) (0.1750,0.000823) (0.2250,0.000760) (0.2750,0.000710) (0.3250,0.000679) (0.3750,0.000651) (0.4250,0.000625) (0.4750,0.000606) (0.5250,0.000591) (0.5750,0.000574) (0.6250,0.000563) (0.6750,0.000551) (0.7250,0.000541) (0.7750,0.000533) (0.8250,0.000525) (0.8750,0.000521) (0.9250,0.000514) (0.9750,0.000512) (1.0250,0.000200) (1.0750,0.000202) (1.1250,0.000193) (1.1750,0.000189) (1.2250,0.000184) (1.2750,0.000182) (1.3250,0.000179) (1.3750,0.000180) (1.4250,0.000179) (1.4750,0.000180) (1.5250,0.000181) (1.5750,0.000180) (1.6250,0.000180) (1.6750,0.000181) (1.7250,0.000182) (1.7750,0.000181) (1.8250,0.000181) (1.8750,0.000183) (1.9250,0.000181) (1.9750,0.000183) (2.0250,0.000181) (2.0750,0.000183) (2.1250,0.000182) (2.1750,0.000182) (2.2250,0.000183) (2.2750,0.000184) (2.3250,0.000182) (2.3750,0.000182) (2.4250,0.000184) (2.4750,0.000183) (2.5250,0.000183) (2.5750,0.000184) (2.6250,0.000185) (2.6750,0.000183) (2.7250,0.000184) (2.7750,0.000183) (2.8250,0.000185) (2.8750,0.000184) (2.9250,0.000184) (2.9750,0.000184) 
      };
      \addlegendentry{MultiRC}

      \addplot[name path=upper2, draw=none, forget plot] coordinates {
       (0.0250,0.001178) (0.0750,0.000695) (0.1250,0.000585) (0.1750,0.000531) (0.2250,0.000494) (0.2750,0.000464) (0.3250,0.000445) (0.3750,0.000427) (0.4250,0.000411) (0.4750,0.000399) (0.5250,0.000388) (0.5750,0.000378) (0.6250,0.000371) (0.6750,0.000363) (0.7250,0.000357) (0.7750,0.000352) (0.8250,0.000346) (0.8750,0.000344) (0.9250,0.000339) (0.9750,0.000338) (1.0250,0.000136) (1.0750,0.000137) (1.1250,0.000130) (1.1750,0.000128) (1.2250,0.000125) (1.2750,0.000124) (1.3250,0.000122) (1.3750,0.000122) (1.4250,0.000122) (1.4750,0.000123) (1.5250,0.000123) (1.5750,0.000122) (1.6250,0.000123) (1.6750,0.000123) (1.7250,0.000123) (1.7750,0.000124) (1.8250,0.000122) (1.8750,0.000125) (1.9250,0.000123) (1.9750,0.000125) (2.0250,0.000122) (2.0750,0.000125) (2.1250,0.000123) (2.1750,0.000124) (2.2250,0.000124) (2.2750,0.000125) (2.3250,0.000123) (2.3750,0.000124) (2.4250,0.000125) (2.4750,0.000124) (2.5250,0.000124) (2.5750,0.000125) (2.6250,0.000125) (2.6750,0.000123) (2.7250,0.000125) (2.7750,0.000123) (2.8250,0.000126) (2.8750,0.000124) (2.9250,0.000125) (2.9750,0.000124)
      };
      \addplot[name path=lower2, draw=none, forget plot] coordinates {
        (0.0250,0.000465) (0.0750,0.000274) (0.1250,0.000231) (0.1750,0.000208) (0.2250,0.000193) (0.2750,0.000180) (0.3250,0.000173) (0.3750,0.000165) (0.4250,0.000159) (0.4750,0.000155) (0.5250,0.000151) (0.5750,0.000147) (0.6250,0.000145) (0.6750,0.000142) (0.7250,0.000140) (0.7750,0.000138) (0.8250,0.000136) (0.8750,0.000135) (0.9250,0.000133) (0.9750,0.000133) (1.0250,0.000054) (1.0750,0.000054) (1.1250,0.000051) (1.1750,0.000051) (1.2250,0.000049) (1.2750,0.000049) (1.3250,0.000048) (1.3750,0.000048) (1.4250,0.000048) (1.4750,0.000048) (1.5250,0.000048) (1.5750,0.000048) (1.6250,0.000048) (1.6750,0.000048) (1.7250,0.000048) (1.7750,0.000049) (1.8250,0.000048) (1.8750,0.000049) (1.9250,0.000048) (1.9750,0.000049) (2.0250,0.000048) (2.0750,0.000049) (2.1250,0.000048) (2.1750,0.000049) (2.2250,0.000049) (2.2750,0.000049) (2.3250,0.000049) (2.3750,0.000049) (2.4250,0.000049) (2.4750,0.000049) (2.5250,0.000049) (2.5750,0.000049) (2.6250,0.000049) (2.6750,0.000048) (2.7250,0.000049) (2.7750,0.000049) (2.8250,0.000050) (2.8750,0.000049) (2.9250,0.000049) (2.9750,0.000049)
      };
      \addplot[blue!40, fill opacity=0.35, forget plot] fill between[of=upper2 and lower2];
      \addplot[blue!80!black, thick] coordinates {
        (0.0250,0.000821) (0.0750,0.000485) (0.1250,0.000408) (0.1750,0.000370) (0.2250,0.000344) (0.2750,0.000322) (0.3250,0.000309) (0.3750,0.000296) (0.4250,0.000285) (0.4750,0.000277) (0.5250,0.000270) (0.5750,0.000263) (0.6250,0.000258) (0.6750,0.000253) (0.7250,0.000248) (0.7750,0.000245) (0.8250,0.000241) (0.8750,0.000240) (0.9250,0.000236) (0.9750,0.000235) (1.0250,0.000095) (1.0750,0.000095) (1.1250,0.000091) (1.1750,0.000089) (1.2250,0.000087) (1.2750,0.000086) (1.3250,0.000085) (1.3750,0.000085) (1.4250,0.000085) (1.4750,0.000085) (1.5250,0.000085) (1.5750,0.000085) (1.6250,0.000085) (1.6750,0.000086) (1.7250,0.000086) (1.7750,0.000086) (1.8250,0.000085) (1.8750,0.000087) (1.9250,0.000086) (1.9750,0.000087) (2.0250,0.000085) (2.0750,0.000087) (2.1250,0.000086) (2.1750,0.000086) (2.2250,0.000087) (2.2750,0.000087) (2.3250,0.000086) (2.3750,0.000086) (2.4250,0.000087) (2.4750,0.000086) (2.5250,0.000087) (2.5750,0.000087) (2.6250,0.000087) (2.6750,0.000086) (2.7250,0.000087) (2.7750,0.000086) (2.8250,0.000088) (2.8750,0.000086) (2.9250,0.000087) (2.9750,0.000086) 
      };
      \addlegendentry{SST-2}

    \end{axis}
  \end{tikzpicture}
\caption{}
\label{sub1}
  \end{subfigure}
  \hfill
  \begin{subfigure}[t]{0.48\linewidth}
    \centering
    \begin{tikzpicture}
      \begin{axis}[
        width=\linewidth, height=5cm,
        xlabel={Gradient Step},
        ylabel={$\frac{\|\m_t\odot \nabla_{\w} \tilde{F}(\w,\btheta)\| }{ \|\nabla_{\w} \tilde{F}(\w,\btheta)\|}$},
        xmin=10, xmax=99,
        ymin=0.1, ymax=1.0,
        grid=major, grid style={dotted,gray!50},
        legend pos=south east,
        tick label style={font=\scriptsize},
        legend style={font=\scriptsize},
        label style={font=\scriptsize},
      ]

      \addplot[draw=none, name path=red_upper, forget plot] coordinates {(0,0.434867) (1,0.480687) (2,0.548762) (3,0.536801) (4,0.582910) (5,0.570433) (6,0.584372) (7,0.578936) (8,0.578902) (9,0.583878) (10,0.578264) (11,0.572935) (12,0.566308) (13,0.563951) (14,0.573789) (15,0.568744) (16,0.565792) (17,0.570737) (18,0.566573) (19,0.562414) (20,0.558564) (21,0.558306) (22,0.556620) (23,0.561429) (24,0.567575) (25,0.568659) (26,0.566783) (27,0.564099) (28,0.562855) (29,0.559920) (30,0.558298) (31,0.556268) (32,0.554045) (33,0.559903) (34,0.558555) (35,0.557559) (36,0.555447) (37,0.553195) (38,0.550995) (39,0.555494) (40,0.555697) (41,0.555985) (42,0.555615) (43,0.559687) (44,0.557891) (45,0.556662) (46,0.554909) (47,0.557585) (48,0.556751) (49,0.555081) (50,0.553665) (51,0.553705) (52,0.552334) (53,0.552917) (54,0.551476) (55,0.552072) (56,0.554235) (57,0.559608) (58,0.558152) (59,0.561017) (60,0.559957) (61,0.559094) (62,0.557857) (63,0.556494) (64,0.556723) (65,0.555462) (66,0.554167) (67,0.553006) (68,0.556504) (69,0.556133) (70,0.555381) (71,0.557329) (72,0.557912) (73,0.561147) (74,0.560421) (75,0.559345) (76,0.558225) (77,0.559623) (78,0.561436) (79,0.560529) (80,0.563088) (81,0.562175) (82,0.563407) (83,0.562322) (84,0.563497) (85,0.562465) (86,0.562545) (87,0.562481) (88,0.561730) (89,0.561739) (90,0.560778) (91,0.561591) (92,0.560969) (93,0.562031) (94,0.562442) (95,0.561502) (96,0.561645) (97,0.560850) (98,0.562260) (99,0.561482)};
    \addplot[draw=none, name path=red_lower, forget plot] coordinates {(0,0.434867) (1,0.434867) (2,0.439433) (3,0.440130) (4,0.446666) (5,0.433131) (6,0.443746) (7,0.407814) (8,0.415948) (9,0.424243) (10,0.425409) (11,0.407439) (12,0.399764) (13,0.403399) (14,0.408683) (15,0.398975) (16,0.401055) (17,0.405766) (18,0.404121) (19,0.399133) (20,0.392638) (21,0.395690) (22,0.382009) (23,0.385701) (24,0.389126) (25,0.392353) (26,0.393755) (27,0.386473) (28,0.388291) (29,0.384411) (30,0.385629) (31,0.386096) (32,0.385923) (33,0.388045) (34,0.389166) (35,0.390525) (36,0.386152) (37,0.383486) (38,0.381485) (39,0.383396) (40,0.385226) (41,0.387026) (42,0.388484) (43,0.390185) (44,0.386811) (45,0.387397) (46,0.384463) (47,0.386215) (48,0.387137) (49,0.386129) (50,0.386039) (51,0.387399) (52,0.383893) (53,0.385371) (54,0.382603) (55,0.384038) (56,0.385494) (57,0.386182) (58,0.384081) (59,0.385415) (60,0.385695) (61,0.386224) (62,0.383379) (63,0.382332) (64,0.383529) (65,0.382933) (66,0.381788) (67,0.381472) (68,0.382382) (69,0.383216) (70,0.379297) (71,0.380486) (72,0.381656) (73,0.382556) (74,0.383009) (75,0.380742) (76,0.378836) (77,0.379988) (78,0.381092) (79,0.381160) (80,0.382102) (81,0.379438) (82,0.380544) (83,0.379129) (84,0.380212) (85,0.378496) (86,0.379406) (87,0.380243) (88,0.380445) (89,0.381285) (90,0.379699) (91,0.380684) (92,0.381016) (93,0.381977) (94,0.382883) (95,0.381888) (96,0.382709) (97,0.380702) (98,0.381588) (99,0.379558)};
    \addplot[red!25, fill opacity=0.35, forget plot] fill between[of=red_upper and red_lower];
    
    \addplot[draw=none, name path=blue_upper, forget plot] coordinates {(0,0.614845) (1,0.614845) (2,0.766987) (3,0.756686) (4,0.781443) (5,0.789419) (6,0.801347) (7,0.797169) (8,0.791054) (9,0.795591) (10,0.789498) (11,0.789610) (12,0.783591) (13,0.792101) (14,0.785314) (15,0.779018) (16,0.773875) (17,0.768226) (18,0.776270) (19,0.770978) (20,0.777982) (21,0.780675) (22,0.776776) (23,0.777622) (24,0.779930) (25,0.782003) (26,0.778497) (27,0.781965) (28,0.779193) (29,0.782786) (30,0.789425) (31,0.789352) (32,0.791903) (33,0.792708) (34,0.789970) (35,0.793521) (36,0.791460) (37,0.796624) (38,0.794027) (39,0.796469) (40,0.799575) (41,0.797805) (42,0.800107) (43,0.797933) (44,0.800972) (45,0.802297) (46,0.802274) (47,0.804138) (48,0.804542) (49,0.803383) (50,0.805249) (51,0.804053) (52,0.805708) (53,0.804664) (54,0.803619) (55,0.803996) (56,0.804971) (57,0.805813) (58,0.807789) (59,0.806474) (60,0.805566) (61,0.804492) (62,0.802720) (63,0.801217) (64,0.803372) (65,0.803152) (66,0.802273) (67,0.803996) (68,0.804254) (69,0.802686) (70,0.803381) (71,0.802564) (72,0.802535) (73,0.804589) (74,0.805294) (75,0.806503) (76,0.807689) (77,0.808329) (78,0.807938) (79,0.807014) (80,0.808281) (81,0.807636) (82,0.807054) (83,0.808079) (84,0.809280) (85,0.809457) (86,0.808990) (87,0.807800) (88,0.807541) (89,0.806450) (90,0.806489) (91,0.805747) (92,0.806554) (93,0.805494) (94,0.805480) (95,0.806133) (96,0.805758) (97,0.805337) (98,0.804542) (99,0.805258)};
    \addplot[draw=none, name path=blue_lower, forget plot] coordinates {(0,0.614845) (1,0.393555) (2,0.436818) (3,0.468388) (4,0.500007) (5,0.523212) (6,0.542961) (7,0.554655) (8,0.562365) (9,0.573562) (10,0.577712) (11,0.585257) (12,0.586584) (13,0.593885) (14,0.588326) (15,0.583379) (16,0.569724) (17,0.560208) (18,0.565752) (19,0.557662) (20,0.562819) (21,0.567650) (22,0.567621) (23,0.571579) (24,0.575648) (25,0.579505) (26,0.571984) (27,0.575734) (28,0.576175) (29,0.579617) (30,0.582653) (31,0.585324) (32,0.588444) (33,0.591214) (34,0.590129) (35,0.592927) (36,0.593351) (37,0.595770) (38,0.592562) (39,0.595111) (40,0.597584) (41,0.590722) (42,0.593157) (43,0.592519) (44,0.594805) (45,0.597011) (46,0.598822) (47,0.583951) (48,0.585968) (49,0.587004) (50,0.589140) (51,0.582104) (52,0.584202) (53,0.585219) (54,0.586177) (55,0.587928) (56,0.589771) (57,0.591556) (58,0.593386) (59,0.588365) (60,0.581960) (61,0.582679) (62,0.580695) (63,0.580541) (64,0.582240) (65,0.583532) (66,0.577820) (67,0.579483) (68,0.580931) (69,0.579976) (70,0.581480) (71,0.576092) (72,0.577371) (73,0.578884) (74,0.580334) (75,0.581816) (76,0.583276) (77,0.584656) (78,0.578677) (79,0.579156) (80,0.580569) (81,0.575451) (82,0.576262) (83,0.577631) (84,0.578992) (85,0.580171) (86,0.581000) (87,0.580660) (88,0.581592) (89,0.578816) (90,0.579871) (91,0.580349) (92,0.581548) (93,0.581370) (94,0.582344) (95,0.583479) (96,0.584225) (97,0.584920) (98,0.585177) (99,0.586260)};
    \addplot[blue!25, fill opacity=0.35, forget plot] fill between[of=blue_upper and blue_lower];

      \addplot[blue, thick] coordinates {
        (0,0.614845) (1,0.504200) (2,0.601902) (3,0.612537) (4,0.640725) (5,0.656315) (6,0.672154) (7,0.675912) (8,0.676710) (9,0.684576) (10,0.683605) (11,0.687433) (12,0.685088) (13,0.692993) (14,0.686820) (15,0.681199) (16,0.671799) (17,0.664217) (18,0.671011) (19,0.664320) (20,0.670400) (21,0.674162) (22,0.672199) (23,0.674601) (24,0.677789) (25,0.680754) (26,0.675241) (27,0.678849) (28,0.677684) (29,0.681202) (30,0.686039) (31,0.687338) (32,0.690173) (33,0.691961) (34,0.690050) (35,0.693224) (36,0.692405) (37,0.696197) (38,0.693295) (39,0.695790) (40,0.698579) (41,0.694264) (42,0.696632) (43,0.695226) (44,0.697888) (45,0.699654) (46,0.700548) (47,0.694045) (48,0.695255) (49,0.695193) (50,0.697195) (51,0.693078) (52,0.694955) (53,0.694941) (54,0.694898) (55,0.695962) (56,0.697371) (57,0.698685) (58,0.700588) (59,0.697420) (60,0.693763) (61,0.693586) (62,0.691708) (63,0.690879) (64,0.692806) (65,0.693342) (66,0.690047) (67,0.691739) (68,0.692592) (69,0.691331) (70,0.692431) (71,0.689328) (72,0.689953) (73,0.691737) (74,0.692814) (75,0.694160) (76,0.695483) (77,0.696492) (78,0.693308) (79,0.693085) (80,0.694425) (81,0.691544) (82,0.691658) (83,0.692855) (84,0.694136) (85,0.694814) (86,0.694995) (87,0.694230) (88,0.694567) (89,0.692633) (90,0.693180) (91,0.693048) (92,0.694051) (93,0.693432) (94,0.693912) (95,0.694806) (96,0.694992) (97,0.695129) (98,0.694859) (99,0.695759)
      };
        \addlegendentry{\texttt{SST2}}

      \addplot[red!80!black, thick] coordinates {
        (0,0.434867) (1,0.457777) (2,0.494097) (3,0.488465) (4,0.514788) (5,0.501782) (6,0.514059) (7,0.493375) (8,0.497425) (9,0.504061) (10,0.501837) (11,0.490187) (12,0.483036) (13,0.483675) (14,0.491236) (15,0.483859) (16,0.483423) (17,0.488251) (18,0.485347) (19,0.480773) (20,0.475601) (21,0.476998) (22,0.469315) (23,0.473565) (24,0.478351) (25,0.480506) (26,0.480269) (27,0.475286) (28,0.475573) (29,0.472165) (30,0.471963) (31,0.471182) (32,0.469984) (33,0.473974) (34,0.473861) (35,0.474042) (36,0.470799) (37,0.468340) (38,0.466240) (39,0.469445) (40,0.470462) (41,0.471505) (42,0.472049) (43,0.474936) (44,0.472351) (45,0.472030) (46,0.469686) (47,0.471900) (48,0.471944) (49,0.470605) (50,0.469852) (51,0.470552) (52,0.468114) (53,0.469144) (54,0.467040) (55,0.468055) (56,0.469864) (57,0.472895) (58,0.471116) (59,0.473216) (60,0.472826) (61,0.472659) (62,0.470618) (63,0.469413) (64,0.470126) (65,0.469198) (66,0.467977) (67,0.467239) (68,0.469443) (69,0.469674) (70,0.467339) (71,0.468907) (72,0.469784) (73,0.471852) (74,0.471715) (75,0.470043) (76,0.468530) (77,0.469805) (78,0.471264) (79,0.470845) (80,0.472595) (81,0.470806) (82,0.471976) (83,0.470726) (84,0.471855) (85,0.470481) (86,0.470975) (87,0.471362) (88,0.471087) (89,0.471512) (90,0.470238) (91,0.471138) (92,0.470992) (93,0.472004) (94,0.472663) (95,0.471695) (96,0.472177) (97,0.470776) (98,0.471924) (99,0.470520)
      };\addlegendentry{\texttt{MultiRC}}
      \end{axis}
    \end{tikzpicture}
    \caption{}
    \label{sub2}
  \end{subfigure}
  \caption{(a) Average per-coordinate gradient magnitude $G_k$ in \eqref{eq:avg_grad_mag}, as a function of the normalized distance to the quantization boundary, $\tilde d_{i,t}^\ell=d_{i,t}^\ell(\w_t^\ell)/\tau_t^\ell$, on \texttt{SST2} dataset using Falcon-E-3B-Instruct model~\citep{tiionebitllms}. Results are averaged over all linear projection layers and $100$ SGD steps. Gradient magnitudes are larger for weights closer to the quantization boundaries. (b) Ratio of the masked latent-weight gradient norm to the full latent-weight gradient norm, $\|\m_t\odot \nabla_{\w_t} \tilde{F}(\w_t,\btheta_t)\| \,/\, \|\nabla_{\w_t} \tilde{F}(\w_t,\btheta_t)\|$, as a function of the gradient steps for a sparsity fraction $\rho_t=0.01$, on {MultiRC} and {SST2} datasets.}
  \label{fig:grad_signal}
\end{figure}

\newpage
\section{TerMeZO Algorithm: Pseudo-code}
\label{app:pseudocode}

\begin{algorithm}[ht]
\caption{TerMeZO: Ternary Me\gls{zo} for BitNet Models}
\label{alg:termezo}
\begin{algorithmic}
\Require Latent weights $\w_0 \in \mathbb{R}^d$; full-precision parameters $\btheta_0 \in \mathbb{R}^p$; learning rate $\eta$; perturbation scale $\epsilon$; initial sparsity fraction $\rho_0$; number of steps $T$; number of perturbations $K$; batch size $B$.
\State Compute thresholds $\tau_0^\ell$ and distances $\{d_{i,0}^\ell(\w_0^\ell)\}_{i,\ell}$ across all layers and set $\xi_0$ as the lower $\rho_0$-quantile of these distances
\State $\mathcal{S}_0^\ell \gets \bigl\{ i : d_{i,0}^\ell(\w_0^\ell) \le \xi_0 \bigr\}$ for each ternary linear projection $\ell$ \Comment{initial sparsity sets}
\For{$t = 0, \ldots, T-1$}
    \If{$t \geq 1$}
        \State Compute thresholds $\tau_t^\ell$ for each $\ell$
        \State $\mathcal{S}_t^\ell \gets \bigl\{ i \in \mathcal{S}_{t-1}^\ell : d_{i,t}^\ell(\w_t^\ell) \le \xi_0  \bigr\}$ for each $\ell$ \Comment{update sparsity sets}
    \EndIf
    \State Form mask $\m_t$ from $\{\mathcal{S}_t^\ell\}_{\ell=1}^L$ 
    \State Sample batch $\mathcal{B}_t$ of size $B$ 
    \For{$k = 1, \ldots, K$}
        \State Sample seed $s_{t,k}$ and use it to draw $\z_\w^{(k)} \sim \mathcal{N}(0, I_d)$ and $\z_\btheta^{(k)} \sim \mathcal{N}(0, I_p)$ 
            \State $\overline{\w}_t \gets \overline{\w}_t + \epsilon \, (\m_t \odot \z_\w^{(k)})$; \quad $\btheta_t \gets \btheta_t + \epsilon \z_\btheta^{(k)}$;\quad $\widetilde{F}^{+} \gets \widehat{\widetilde{F}}(\w_t, \btheta_t; \mathcal{B}_t)$ 
            \State $\overline{\w}_t\gets \overline{\w}_t - 2\epsilon \, (\m_t \odot \z_\w^{(k)})$; \quad $\btheta_t \gets \btheta_t - 2\epsilon \z_\btheta^{(k)}$;\quad $\widetilde{F}^{-} \gets \widehat{\widetilde{F}}(\w_t, \btheta_t; \mathcal{B}_t)$ 
            \State $\overline{\w}_t \gets \overline{\w}_t + \epsilon \, (\m_t \odot \z_\w^{(k)})$; \quad $\btheta_t \gets \btheta_t + \epsilon \z_\btheta^{(k)}$ \Comment{restore}
            \State $g^{(k)} \gets \frac{\widetilde{F}^{+} - \widetilde{F}^{-}}{2\epsilon}$ \Comment{projected gradient}
    \EndFor
    \For{$k = 1, \ldots, K$}
        \State Re-draw $\z_\w^{(k)}$ and $\z_\btheta^{(k)}$ from seed $s_{t,k}$ \Comment{regenerate perturbations}
        \State $\overline{\w}_{t+1} \gets \overline{\w}_t - \tfrac{\eta}{K} \, g^{(k)} \cdot (\m_t \odot \z_\w^{(k)})$;\quad $\btheta_{t+1} \gets \btheta_{t} - \tfrac{\eta}{K} \, g^{(k)} \, \z_\btheta^{(k)}$ \Comment{weight update}
    \EndFor
\EndFor
\State \Return $\w_T, \btheta_T$
\end{algorithmic}
\end{algorithm}

\newpage
\section{TerMeZO for RL-based Fine-tuning}
\label{app:rl_app}
A recent line of work has shown that \gls{zo} methods are effective for RL-based fine-tuning of \gls{llms}~\citep{qiu2025evolution,xu2026quantized}. Closely related to Me\gls{zo} with multiple perturbations, the approach applies several perturbations, estimates the corresponding rewards, and updates the model with the reward-weighted sum of those perturbations. TerMe\gls{zo} extends naturally to this setting by restricting latent weights to which the \gls{es} algorithm is applied. We report results below assessing the effectiveness of TerMe\gls{zo} in this setting, in terms of both convergence speed and final accuracy.

Figure~\ref{fig:es_countdown_2} shows the train and test reward curves for fine-tuning Falcon-E-3B-Instruct~\citep{tiionebitllms} on the Countdown task \citep{gandhi2024stream} using \gls{rlvr}, following the same settings and hyperparameters as in~\citep{qiu2025evolution}. The reward reflects both answer correctness and output format, with the former weighted at $0.9$. Unlike the GLUE classification tasks, QZO clearly struggles here, as restricting updates to the ternary scales is not enough to solve a task of this difficulty. TerMeZO, in contrast, outperforms full-parameter MeZO and converges noticeably faster.

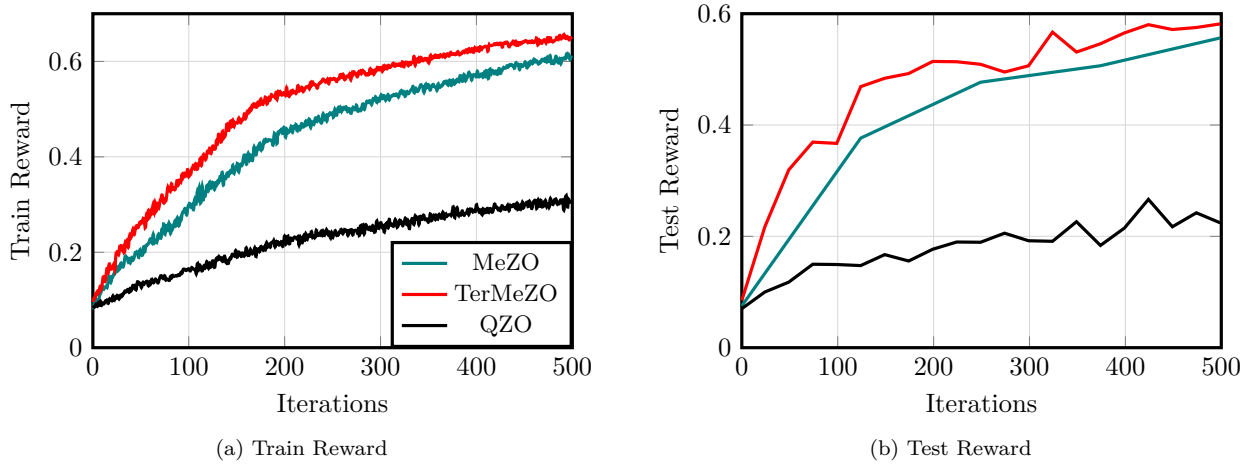
\begin{figure}[h]
    \centering
    \begin{subfigure}{0.48\textwidth}
        \centering
\begin{tikzpicture}
\begin{axis}[
    width=\textwidth,
    height=6cm,
    xlabel={Iterations},
    ylabel={Train Reward},
    xmin=0,
    xmax=500,
    ymin=0,
    ymax=0.7,
    legend style={at={(1.001,-0.001)}, anchor=south east, font=\small},
    ymajorgrids=true,
    xmajorgrids=true,
    grid style={gray!30},
    line width=1.2pt,
]

\addplot[color=teal, mark=none, mark size=2pt, smooth, tension=1.0]coordinates{

(0,0.0844)
(1,0.0874)
(2,0.1004)
(3,0.1024)
(4,0.0936)
(5,0.0985)
(6,0.0978)
(7,0.1100)
(8,0.1136)
(9,0.1195)
(10,0.1273)
(11,0.1227)
(12,0.1247)
(13,0.1263)
(14,0.1294)
(15,0.1267)
(16,0.1412)
(17,0.1502)
(18,0.1399)
(19,0.1382)
(20,0.1486)
(21,0.1489)
(22,0.1553)
(23,0.1554)
(24,0.1670)
(25,0.1658)
(26,0.1608)
(27,0.1707)
(28,0.1651)
(29,0.1768)
(30,0.1700)
(31,0.1738)
(32,0.1711)
(33,0.1841)
(34,0.1918)
(35,0.1883)
(36,0.1880)
(37,0.1791)
(38,0.1856)
(39,0.1774)
(40,0.1881)
(41,0.1826)
(42,0.1865)
(43,0.1883)
(44,0.1978)
(45,0.2000)
(46,0.2004)
(47,0.1947)
(48,0.1899)
(49,0.2048)
(50,0.1965)
(51,0.1969)
(52,0.1952)
(53,0.2116)
(54,0.2007)
(55,0.2095)
(56,0.2148)
(57,0.2157)
(58,0.2185)
(59,0.2234)
(60,0.2283)
(61,0.2251)
(62,0.2199)
(63,0.2378)
(64,0.2208)
(65,0.2251)
(66,0.2224)
(67,0.2231)
(68,0.2362)
(69,0.2316)
(70,0.2381)
(71,0.2298)
(72,0.2480)
(73,0.2433)
(74,0.2592)
(75,0.2376)
(76,0.2526)
(77,0.2543)
(78,0.2462)
(79,0.2596)
(80,0.2562)
(81,0.2395)
(82,0.2531)
(83,0.2483)
(84,0.2555)
(85,0.2609)
(86,0.2660)
(87,0.2557)
(88,0.2720)
(89,0.2644)
(90,0.2564)
(91,0.2706)
(92,0.2673)
(93,0.2792)
(94,0.2847)
(95,0.2850)
(96,0.2917)
(97,0.2848)
(98,0.2791)
(99,0.2932)
(100,0.2993)
(101,0.3023)
(102,0.2956)
(103,0.2961)
(104,0.2996)
(105,0.2894)
(106,0.3099)
(107,0.3025)
(108,0.3052)
(109,0.3276)
(110,0.3110)
(111,0.3358)
(112,0.3077)
(113,0.3093)
(114,0.3388)
(115,0.3104)
(116,0.3234)
(117,0.3172)
(118,0.3283)
(119,0.3398)
(120,0.3282)
(121,0.3358)
(122,0.3473)
(123,0.3353)
(124,0.3411)
(125,0.3464)
(126,0.3381)
(127,0.3319)
(128,0.3456)
(129,0.3402)
(130,0.3492)
(131,0.3454)
(132,0.3537)
(133,0.3590)
(134,0.3540)
(135,0.3670)
(136,0.3412)
(137,0.3561)
(138,0.3586)
(139,0.3558)
(140,0.3660)
(141,0.3601)
(142,0.3634)
(143,0.3698)
(144,0.3698)
(145,0.3690)
(146,0.3822)
(147,0.3656)
(148,0.3757)
(149,0.3726)
(150,0.3811)
(151,0.3849)
(152,0.3690)
(153,0.3777)
(154,0.3846)
(155,0.3812)
(156,0.3855)
(157,0.3977)
(158,0.4038)
(159,0.3980)
(160,0.3848)
(161,0.3919)
(162,0.4076)
(163,0.3932)
(164,0.3994)
(165,0.4011)
(166,0.3922)
(167,0.3991)
(168,0.4201)
(169,0.4114)
(170,0.4040)
(171,0.4117)
(172,0.4212)
(173,0.4182)
(174,0.4282)
(175,0.4148)
(176,0.4210)
(177,0.4182)
(178,0.4321)
(179,0.4292)
(180,0.4366)
(181,0.4339)
(182,0.4323)
(183,0.4355)
(184,0.4356)
(185,0.4352)
(186,0.4347)
(187,0.4308)
(188,0.4258)
(189,0.4458)
(190,0.4475)
(191,0.4428)
(192,0.4279)
(193,0.4503)
(194,0.4491)
(195,0.4536)
(196,0.4556)
(197,0.4576)
(198,0.4496)
(199,0.4616)
(200,0.4548)
(201,0.4410)
(202,0.4573)
(203,0.4525)
(204,0.4507)
(205,0.4545)
(206,0.4566)
(207,0.4638)
(208,0.4516)
(209,0.4590)
(210,0.4645)
(211,0.4601)
(212,0.4675)
(213,0.4547)
(214,0.4628)
(215,0.4598)
(216,0.4570)
(217,0.4656)
(218,0.4695)
(219,0.4627)
(220,0.4645)
(221,0.4663)
(222,0.4748)
(223,0.4674)
(224,0.4752)
(225,0.4684)
(226,0.4785)
(227,0.4620)
(228,0.4700)
(229,0.4732)
(230,0.4680)
(231,0.4712)
(232,0.4809)
(233,0.4777)
(234,0.4744)
(235,0.4843)
(236,0.4804)
(237,0.4763)
(238,0.4712)
(239,0.4742)
(240,0.4898)
(241,0.4835)
(242,0.4781)
(243,0.4865)
(244,0.4751)
(245,0.4871)
(246,0.4870)
(247,0.4916)
(248,0.4915)
(249,0.4893)
(250,0.4914)
(251,0.4833)
(252,0.4975)
(253,0.4961)
(254,0.4874)
(255,0.4943)
(256,0.5035)
(257,0.4963)
(258,0.5045)
(259,0.4958)
(260,0.4973)
(261,0.5027)
(262,0.5028)
(263,0.5002)
(264,0.5050)
(265,0.4981)
(266,0.4953)
(267,0.5013)
(268,0.4970)
(269,0.4914)
(270,0.4954)
(271,0.5028)
(272,0.5027)
(273,0.4986)
(274,0.5022)
(275,0.5081)
(276,0.4976)
(277,0.4959)
(278,0.4997)
(279,0.4967)
(280,0.4995)
(281,0.5088)
(282,0.5092)
(283,0.4994)
(284,0.5052)
(285,0.5129)
(286,0.5168)
(287,0.5137)
(288,0.5169)
(289,0.5168)
(290,0.5116)
(291,0.5108)
(292,0.5203)
(293,0.5212)
(294,0.5190)
(295,0.5256)
(296,0.5143)
(297,0.5207)
(298,0.5194)
(299,0.5174)
(300,0.5267)
(301,0.5294)
(302,0.5225)
(303,0.5257)
(304,0.5277)
(305,0.5195)
(306,0.5206)
(307,0.5255)
(308,0.5301)
(309,0.5197)
(310,0.5242)
(311,0.5271)
(312,0.5270)
(313,0.5294)
(314,0.5306)
(315,0.5284)
(316,0.5298)
(317,0.5314)
(318,0.5273)
(319,0.5262)
(320,0.5340)
(321,0.5358)
(322,0.5390)
(323,0.5324)
(324,0.5370)
(325,0.5301)
(326,0.5275)
(327,0.5363)
(328,0.5348)
(329,0.5323)
(330,0.5350)
(331,0.5381)
(332,0.5406)
(333,0.5436)
(334,0.5469)
(335,0.5345)
(336,0.5421)
(337,0.5433)
(338,0.5380)
(339,0.5376)
(340,0.5416)
(341,0.5469)
(342,0.5511)
(343,0.5485)
(344,0.5459)
(345,0.5360)
(346,0.5390)
(347,0.5519)
(348,0.5452)
(349,0.5479)
(350,0.5510)
(351,0.5437)
(352,0.5402)
(353,0.5480)
(354,0.5483)
(355,0.5519)
(356,0.5474)
(357,0.5465)
(358,0.5451)
(359,0.5397)
(360,0.5526)
(361,0.5458)
(362,0.5473)
(363,0.5475)
(364,0.5506)
(365,0.5510)
(366,0.5571)
(367,0.5547)
(368,0.5580)
(369,0.5576)
(370,0.5510)
(371,0.5535)
(372,0.5607)
(373,0.5539)
(374,0.5500)
(375,0.5604)
(376,0.5579)
(377,0.5522)
(378,0.5587)
(379,0.5569)
(380,0.5626)
(381,0.5582)
(382,0.5591)
(383,0.5702)
(384,0.5617)
(385,0.5587)
(386,0.5608)
(387,0.5727)
(388,0.5629)
(389,0.5601)
(390,0.5676)
(391,0.5680)
(392,0.5770)
(393,0.5719)
(394,0.5711)
(395,0.5609)
(396,0.5708)
(397,0.5701)
(398,0.5717)
(399,0.5704)
(400,0.5684)
(401,0.5793)
(402,0.5694)
(403,0.5677)
(404,0.5702)
(405,0.5736)
(406,0.5713)
(407,0.5671)
(408,0.5764)
(409,0.5686)
(410,0.5804)
(411,0.5756)
(412,0.5655)
(413,0.5675)
(414,0.5838)
(415,0.5733)
(416,0.5795)
(417,0.5810)
(418,0.5776)
(419,0.5742)
(420,0.5792)
(421,0.5788)
(422,0.5728)
(423,0.5810)
(424,0.5846)
(425,0.5754)
(426,0.5785)
(427,0.5819)
(428,0.5868)
(429,0.5865)
(430,0.5858)
(431,0.5871)
(432,0.5910)
(433,0.5879)
(434,0.5922)
(435,0.5915)
(436,0.5926)
(437,0.5931)
(438,0.5899)
(439,0.5877)
(440,0.5866)
(441,0.5966)
(442,0.5951)
(443,0.5930)
(444,0.5918)
(445,0.5830)
(446,0.5929)
(447,0.5908)
(448,0.5894)
(449,0.5916)
(450,0.5952)
(451,0.5989)
(452,0.5912)
(453,0.5869)
(454,0.5966)
(455,0.5918)
(456,0.5983)
(457,0.5998)
(458,0.5955)
(459,0.5974)
(460,0.6049)
(461,0.5993)
(462,0.6017)
(463,0.6006)
(464,0.5964)
(465,0.5950)
(466,0.5992)
(467,0.5979)
(468,0.6033)
(469,0.6043)
(470,0.6019)
(471,0.6015)
(472,0.6065)
(473,0.5990)
(474,0.6084)
(475,0.5999)
(476,0.6066)
(477,0.6115)
(478,0.5968)
(479,0.5995)
(480,0.5999)
(481,0.5961)
(482,0.6046)
(483,0.6108)
(484,0.6027)
(485,0.5977)
(486,0.5995)
(487,0.6057)
(488,0.5949)
(489,0.6056)
(490,0.6043)
(491,0.6066)
(492,0.6073)
(493,0.6059)
(494,0.6158)
(495,0.6107)
(496,0.6154)
(497,0.6064)
(498,0.6048)
(499,0.6124)

};\addlegendentry{MeZO}

\addplot[color=red, mark=none, mark size=2pt,smooth, tension=1.0] coordinates {

(0,0.0957)
(1,0.1028)
(2,0.1043)
(3,0.1128)
(4,0.1107)
(5,0.1186)
(6,0.1141)
(7,0.1236)
(8,0.1288)
(9,0.1233)
(10,0.1381)
(11,0.1485)
(12,0.1401)
(13,0.1412)
(14,0.1568)
(15,0.1662)
(16,0.1565)
(17,0.1751)
(18,0.1663)
(19,0.1647)
(20,0.1665)
(21,0.1720)
(22,0.1734)
(23,0.1638)
(24,0.1932)
(25,0.1998)
(26,0.1998)
(27,0.1956)
(28,0.2115)
(29,0.2105)
(30,0.2122)
(31,0.2032)
(32,0.2131)
(33,0.1998)
(34,0.2213)
(35,0.2185)
(36,0.2239)
(37,0.2105)
(38,0.2324)
(39,0.2279)
(40,0.2392)
(41,0.2361)
(42,0.2492)
(43,0.2375)
(44,0.2453)
(45,0.2544)
(46,0.2535)
(47,0.2515)
(48,0.2560)
(49,0.2628)
(50,0.2600)
(51,0.2637)
(52,0.2680)
(53,0.2701)
(54,0.2587)
(55,0.2537)
(56,0.2797)
(57,0.2685)
(58,0.2878)
(59,0.2762)
(60,0.2787)
(61,0.2893)
(62,0.2813)
(63,0.3007)
(64,0.2831)
(65,0.3041)
(66,0.2862)
(67,0.2999)
(68,0.2939)
(69,0.2942)
(70,0.3068)
(71,0.3113)
(72,0.3200)
(73,0.3123)
(74,0.3032)
(75,0.3124)
(76,0.3073)
(77,0.3183)
(78,0.3344)
(79,0.3342)
(80,0.3288)
(81,0.3295)
(82,0.3435)
(83,0.3328)
(84,0.3345)
(85,0.3311)
(86,0.3330)
(87,0.3424)
(88,0.3502)
(89,0.3409)
(90,0.3501)
(91,0.3475)
(92,0.3489)
(93,0.3633)
(94,0.3558)
(95,0.3647)
(96,0.3604)
(97,0.3605)
(98,0.3679)
(99,0.3725)
(100,0.3584)
(101,0.3692)
(102,0.3632)
(103,0.3738)
(104,0.3770)
(105,0.3821)
(106,0.3744)
(107,0.3853)
(108,0.3807)
(109,0.3742)
(110,0.3844)
(111,0.3847)
(112,0.3964)
(113,0.3823)
(114,0.3910)
(115,0.4033)
(116,0.4020)
(117,0.3970)
(118,0.3956)
(119,0.4078)
(120,0.4125)
(121,0.4104)
(122,0.4178)
(123,0.4068)
(124,0.4151)
(125,0.4225)
(126,0.4187)
(127,0.4161)
(128,0.4200)
(129,0.4346)
(130,0.4270)
(131,0.4468)
(132,0.4376)
(133,0.4404)
(134,0.4318)
(135,0.4528)
(136,0.4395)
(137,0.4437)
(138,0.4443)
(139,0.4582)
(140,0.4689)
(141,0.4573)
(142,0.4579)
(143,0.4727)
(144,0.4704)
(145,0.4644)
(146,0.4764)
(147,0.4636)
(148,0.4706)
(149,0.4711)
(150,0.4714)
(151,0.4724)
(152,0.4805)
(153,0.4788)
(154,0.4715)
(155,0.4758)
(156,0.4786)
(157,0.4789)
(158,0.4897)
(159,0.4816)
(160,0.4841)
(161,0.4949)
(162,0.4921)
(163,0.4937)
(164,0.5020)
(165,0.4976)
(166,0.5074)
(167,0.5045)
(168,0.5110)
(169,0.5058)
(170,0.5105)
(171,0.5042)
(172,0.5142)
(173,0.5120)
(174,0.5094)
(175,0.5208)
(176,0.5149)
(177,0.5100)
(178,0.5131)
(179,0.5184)
(180,0.5249)
(181,0.5148)
(182,0.5118)
(183,0.5230)
(184,0.5190)
(185,0.5203)
(186,0.5300)
(187,0.5274)
(188,0.5268)
(189,0.5264)
(190,0.5261)
(191,0.5332)
(192,0.5203)
(193,0.5371)
(194,0.5296)
(195,0.5358)
(196,0.5344)
(197,0.5294)
(198,0.5332)
(199,0.5323)
(200,0.5337)
(201,0.5297)
(202,0.5451)
(203,0.5306)
(204,0.5309)
(205,0.5274)
(206,0.5366)
(207,0.5325)
(208,0.5394)
(209,0.5352)
(210,0.5352)
(211,0.5359)
(212,0.5330)
(213,0.5408)
(214,0.5493)
(215,0.5399)
(216,0.5396)
(217,0.5408)
(218,0.5457)
(219,0.5379)
(220,0.5382)
(221,0.5421)
(222,0.5452)
(223,0.5471)
(224,0.5470)
(225,0.5474)
(226,0.5431)
(227,0.5557)
(228,0.5492)
(229,0.5557)
(230,0.5541)
(231,0.5514)
(232,0.5510)
(233,0.5489)
(234,0.5439)
(235,0.5497)
(236,0.5536)
(237,0.5547)
(238,0.5574)
(239,0.5585)
(240,0.5629)
(241,0.5601)
(242,0.5583)
(243,0.5569)
(244,0.5609)
(245,0.5617)
(246,0.5534)
(247,0.5631)
(248,0.5621)
(249,0.5621)
(250,0.5626)
(251,0.5645)
(252,0.5631)
(253,0.5625)
(254,0.5581)
(255,0.5606)
(256,0.5640)
(257,0.5640)
(258,0.5724)
(259,0.5669)
(260,0.5697)
(261,0.5729)
(262,0.5641)
(263,0.5619)
(264,0.5662)
(265,0.5692)
(266,0.5681)
(267,0.5685)
(268,0.5734)
(269,0.5615)
(270,0.5750)
(271,0.5710)
(272,0.5780)
(273,0.5760)
(274,0.5741)
(275,0.5707)
(276,0.5682)
(277,0.5773)
(278,0.5745)
(279,0.5802)
(280,0.5741)
(281,0.5811)
(282,0.5775)
(283,0.5785)
(284,0.5738)
(285,0.5715)
(286,0.5737)
(287,0.5740)
(288,0.5732)
(289,0.5843)
(290,0.5784)
(291,0.5759)
(292,0.5812)
(293,0.5728)
(294,0.5811)
(295,0.5789)
(296,0.5838)
(297,0.5876)
(298,0.5823)
(299,0.5855)
(300,0.5869)
(301,0.5846)
(302,0.5794)
(303,0.5907)
(304,0.5842)
(305,0.5896)
(306,0.5880)
(307,0.5847)
(308,0.5914)
(309,0.5898)
(310,0.5874)
(311,0.5902)
(312,0.5918)
(313,0.5960)
(314,0.5942)
(315,0.5897)
(316,0.5944)
(317,0.5917)
(318,0.5929)
(319,0.5925)
(320,0.5949)
(321,0.5908)
(322,0.5923)
(323,0.5989)
(324,0.5893)
(325,0.5921)
(326,0.5980)
(327,0.5959)
(328,0.5976)
(329,0.5910)
(330,0.5978)
(331,0.5937)
(332,0.5960)
(333,0.5982)
(334,0.6002)
(335,0.6102)
(336,0.5926)
(337,0.5996)
(338,0.6047)
(339,0.5992)
(340,0.6074)
(341,0.5999)
(342,0.6070)
(343,0.6065)
(344,0.6014)
(345,0.6043)
(346,0.6040)
(347,0.6063)
(348,0.6022)
(349,0.6049)
(350,0.5975)
(351,0.6007)
(352,0.6126)
(353,0.6062)
(354,0.6062)
(355,0.6102)
(356,0.6086)
(357,0.6133)
(358,0.6082)
(359,0.6115)
(360,0.6118)
(361,0.6083)
(362,0.6133)
(363,0.6125)
(364,0.6117)
(365,0.6103)
(366,0.6153)
(367,0.6095)
(368,0.6103)
(369,0.6148)
(370,0.6154)
(371,0.6115)
(372,0.6126)
(373,0.6189)
(374,0.6152)
(375,0.6161)
(376,0.6156)
(377,0.6146)
(378,0.6160)
(379,0.6148)
(380,0.6191)
(381,0.6197)
(382,0.6145)
(383,0.6146)
(384,0.6170)
(385,0.6172)
(386,0.6211)
(387,0.6198)
(388,0.6192)
(389,0.6236)
(390,0.6234)
(391,0.6187)
(392,0.6265)
(393,0.6216)
(394,0.6211)
(395,0.6260)
(396,0.6253)
(397,0.6190)
(398,0.6277)
(399,0.6273)
(400,0.6227)
(401,0.6282)
(402,0.6249)
(403,0.6219)
(404,0.6307)
(405,0.6248)
(406,0.6282)
(407,0.6292)
(408,0.6304)
(409,0.6317)
(410,0.6310)
(411,0.6288)
(412,0.6214)
(413,0.6334)
(414,0.6338)
(415,0.6331)
(416,0.6334)
(417,0.6343)
(418,0.6349)
(419,0.6372)
(420,0.6350)
(421,0.6374)
(422,0.6376)
(423,0.6350)
(424,0.6383)
(425,0.6361)
(426,0.6374)
(427,0.6300)
(428,0.6363)
(429,0.6383)
(430,0.6399)
(431,0.6382)
(432,0.6404)
(433,0.6360)
(434,0.6307)
(435,0.6343)
(436,0.6369)
(437,0.6369)
(438,0.6419)
(439,0.6411)
(440,0.6354)
(441,0.6420)
(442,0.6338)
(443,0.6391)
(444,0.6346)
(445,0.6393)
(446,0.6370)
(447,0.6374)
(448,0.6407)
(449,0.6398)
(450,0.6421)
(451,0.6432)
(452,0.6431)
(453,0.6406)
(454,0.6391)
(455,0.6419)
(456,0.6411)
(457,0.6398)
(458,0.6432)
(459,0.6398)
(460,0.6411)
(461,0.6414)
(462,0.6456)
(463,0.6481)
(464,0.6344)
(465,0.6443)
(466,0.6391)
(467,0.6428)
(468,0.6379)
(469,0.6421)
(470,0.6442)
(471,0.6449)
(472,0.6408)
(473,0.6399)
(474,0.6494)
(475,0.6465)
(476,0.6415)
(477,0.6495)
(478,0.6438)
(479,0.6421)
(480,0.6409)
(481,0.6436)
(482,0.6443)
(483,0.6466)
(484,0.6507)
(485,0.6456)
(486,0.6490)
(487,0.6546)
(488,0.6483)
(489,0.6486)
(490,0.6502)
(491,0.6559)
(492,0.6540)
(493,0.6481)
(494,0.6516)
(495,0.6518)
(496,0.6481)
(497,0.6452)
(498,0.6498)
(499,0.6469)
};\addlegendentry{TerMeZO}

\addplot[color=black, mark=none, mark size=2pt,smooth, tension=1.0] coordinates {
(0,0.0835)
(1,0.0851)
(2,0.0858)
(3,0.0907)
(4,0.0905)
(5,0.0870)
(6,0.0947)
(7,0.0896)
(8,0.0975)
(9,0.0892)
(10,0.0927)
(11,0.0929)
(12,0.0928)
(13,0.0976)
(14,0.0984)
(15,0.0988)
(16,0.0929)
(17,0.1008)
(18,0.1010)
(19,0.1015)
(20,0.0989)
(21,0.1086)
(22,0.0992)
(23,0.1052)
(24,0.1028)
(25,0.1088)
(26,0.1086)
(27,0.1026)
(28,0.1041)
(29,0.1103)
(30,0.1070)
(31,0.1119)
(32,0.1124)
(33,0.1176)
(34,0.1179)
(35,0.1175)
(36,0.1216)
(37,0.1199)
(38,0.1163)
(39,0.1240)
(40,0.1194)
(41,0.1268)
(42,0.1300)
(43,0.1250)
(44,0.1302)
(45,0.1213)
(46,0.1339)
(47,0.1382)
(48,0.1336)
(49,0.1348)
(50,0.1288)
(51,0.1273)
(52,0.1365)
(53,0.1328)
(54,0.1347)
(55,0.1365)
(56,0.1392)
(57,0.1370)
(58,0.1297)
(59,0.1370)
(60,0.1325)
(61,0.1423)
(62,0.1407)
(63,0.1365)
(64,0.1377)
(65,0.1436)
(66,0.1414)
(67,0.1401)
(68,0.1358)
(69,0.1418)
(70,0.1474)
(71,0.1508)
(72,0.1485)
(73,0.1452)
(74,0.1455)
(75,0.1506)
(76,0.1420)
(77,0.1440)
(78,0.1427)
(79,0.1427)
(80,0.1416)
(81,0.1493)
(82,0.1460)
(83,0.1509)
(84,0.1502)
(85,0.1495)
(86,0.1506)
(87,0.1511)
(88,0.1504)
(89,0.1528)
(90,0.1512)
(91,0.1536)
(92,0.1581)
(93,0.1560)
(94,0.1630)
(95,0.1647)
(96,0.1636)
(97,0.1618)
(98,0.1628)
(99,0.1647)
(100,0.1648)
(101,0.1615)
(102,0.1619)
(103,0.1678)
(104,0.1658)
(105,0.1651)
(106,0.1635)
(107,0.1627)
(108,0.1558)
(109,0.1761)
(110,0.1633)
(111,0.1701)
(112,0.1717)
(113,0.1643)
(114,0.1610)
(115,0.1656)
(116,0.1721)
(117,0.1669)
(118,0.1738)
(119,0.1802)
(120,0.1695)
(121,0.1675)
(122,0.1767)
(123,0.1689)
(124,0.1760)
(125,0.1820)
(126,0.1770)
(127,0.1684)
(128,0.1800)
(129,0.1875)
(130,0.1841)
(131,0.1905)
(132,0.1892)
(133,0.1847)
(134,0.1883)
(135,0.1816)
(136,0.1824)
(137,0.1913)
(138,0.1875)
(139,0.1781)
(140,0.1821)
(141,0.1765)
(142,0.1819)
(143,0.1876)
(144,0.1872)
(145,0.1943)
(146,0.1954)
(147,0.1929)
(148,0.1936)
(149,0.2037)
(150,0.1937)
(151,0.1905)
(152,0.1998)
(153,0.1978)
(154,0.1960)
(155,0.1914)
(156,0.1985)
(157,0.1932)
(158,0.1925)
(159,0.1983)
(160,0.2093)
(161,0.1974)
(162,0.1963)
(163,0.1940)
(164,0.2045)
(165,0.2033)
(166,0.1966)
(167,0.2009)
(168,0.2023)
(169,0.2086)
(170,0.1913)
(171,0.2036)
(172,0.2052)
(173,0.2061)
(174,0.2014)
(175,0.2113)
(176,0.2096)
(177,0.2066)
(178,0.2054)
(179,0.2119)
(180,0.2049)
(181,0.2196)
(182,0.2136)
(183,0.2119)
(184,0.2168)
(185,0.2106)
(186,0.2187)
(187,0.2208)
(188,0.2229)
(189,0.2232)
(190,0.2112)
(191,0.2213)
(192,0.2170)
(193,0.2218)
(194,0.2161)
(195,0.2163)
(196,0.2197)
(197,0.2212)
(198,0.2122)
(199,0.2214)
(200,0.2341)
(201,0.2222)
(202,0.2191)
(203,0.2326)
(204,0.2287)
(205,0.2177)
(206,0.2262)
(207,0.2293)
(208,0.2249)
(209,0.2281)
(210,0.2315)
(211,0.2311)
(212,0.2302)
(213,0.2278)
(214,0.2316)
(215,0.2151)
(216,0.2317)
(217,0.2250)
(218,0.2199)
(219,0.2242)
(220,0.2349)
(221,0.2263)
(222,0.2391)
(223,0.2274)
(224,0.2298)
(225,0.2353)
(226,0.2359)
(227,0.2342)
(228,0.2432)
(229,0.2312)
(230,0.2337)
(231,0.2329)
(232,0.2325)
(233,0.2462)
(234,0.2358)
(235,0.2381)
(236,0.2385)
(237,0.2361)
(238,0.2438)
(239,0.2359)
(240,0.2448)
(241,0.2356)
(242,0.2385)
(243,0.2447)
(244,0.2466)
(245,0.2383)
(246,0.2362)
(247,0.2407)
(248,0.2297)
(249,0.2417)
(250,0.2311)
(251,0.2372)
(252,0.2410)
(253,0.2448)
(254,0.2359)
(255,0.2380)
(256,0.2364)
(257,0.2425)
(258,0.2425)
(259,0.2480)
(260,0.2410)
(261,0.2396)
(262,0.2465)
(263,0.2467)
(264,0.2446)
(265,0.2449)
(266,0.2352)
(267,0.2444)
(268,0.2416)
(269,0.2465)
(270,0.2458)
(271,0.2363)
(272,0.2440)
(273,0.2457)
(274,0.2458)
(275,0.2448)
(276,0.2403)
(277,0.2485)
(278,0.2354)
(279,0.2553)
(280,0.2423)
(281,0.2505)
(282,0.2566)
(283,0.2386)
(284,0.2544)
(285,0.2571)
(286,0.2513)
(287,0.2481)
(288,0.2494)
(289,0.2524)
(290,0.2460)
(291,0.2511)
(292,0.2555)
(293,0.2459)
(294,0.2423)
(295,0.2483)
(296,0.2651)
(297,0.2550)
(298,0.2503)
(299,0.2532)
(300,0.2616)
(301,0.2582)
(302,0.2509)
(303,0.2536)
(304,0.2424)
(305,0.2574)
(306,0.2576)
(307,0.2608)
(308,0.2512)
(309,0.2687)
(310,0.2591)
(311,0.2570)
(312,0.2587)
(313,0.2545)
(314,0.2577)
(315,0.2673)
(316,0.2653)
(317,0.2541)
(318,0.2541)
(319,0.2597)
(320,0.2674)
(321,0.2609)
(322,0.2638)
(323,0.2558)
(324,0.2599)
(325,0.2652)
(326,0.2710)
(327,0.2602)
(328,0.2587)
(329,0.2653)
(330,0.2496)
(331,0.2618)
(332,0.2634)
(333,0.2707)
(334,0.2647)
(335,0.2592)
(336,0.2755)
(337,0.2708)
(338,0.2701)
(339,0.2724)
(340,0.2579)
(341,0.2724)
(342,0.2744)
(343,0.2687)
(344,0.2734)
(345,0.2777)
(346,0.2704)
(347,0.2681)
(348,0.2701)
(349,0.2667)
(350,0.2632)
(351,0.2706)
(352,0.2801)
(353,0.2670)
(354,0.2711)
(355,0.2755)
(356,0.2771)
(357,0.2688)
(358,0.2778)
(359,0.2708)
(360,0.2706)
(361,0.2792)
(362,0.2757)
(363,0.2741)
(364,0.2724)
(365,0.2839)
(366,0.2728)
(367,0.2715)
(368,0.2802)
(369,0.2760)
(370,0.2769)
(371,0.2781)
(372,0.2789)
(373,0.2788)
(374,0.2786)
(375,0.2877)
(376,0.2941)
(377,0.2744)
(378,0.2782)
(379,0.2869)
(380,0.2929)
(381,0.2784)
(382,0.2791)
(383,0.2808)
(384,0.2785)
(385,0.2864)
(386,0.2831)
(387,0.2775)
(388,0.2800)
(389,0.2822)
(390,0.2839)
(391,0.2890)
(392,0.2789)
(393,0.2893)
(394,0.2872)
(395,0.2904)
(396,0.2850)
(397,0.2852)
(398,0.2877)
(399,0.2862)
(400,0.2994)
(401,0.2905)
(402,0.2904)
(403,0.2837)
(404,0.2855)
(405,0.2885)
(406,0.2788)
(407,0.2909)
(408,0.2825)
(409,0.2887)
(410,0.2839)
(411,0.2914)
(412,0.2891)
(413,0.2831)
(414,0.2903)
(415,0.2909)
(416,0.2867)
(417,0.2848)
(418,0.2936)
(419,0.3011)
(420,0.2949)
(421,0.2950)
(422,0.2911)
(423,0.2890)
(424,0.2923)
(425,0.2870)
(426,0.2932)
(427,0.2974)
(428,0.2915)
(429,0.2928)
(430,0.2899)
(431,0.2922)
(432,0.2969)
(433,0.2975)
(434,0.2872)
(435,0.2968)
(436,0.2972)
(437,0.2956)
(438,0.2975)
(439,0.3002)
(440,0.2970)
(441,0.2928)
(442,0.3026)
(443,0.2903)
(444,0.2992)
(445,0.2910)
(446,0.2972)
(447,0.2919)
(448,0.3051)
(449,0.2938)
(450,0.2977)
(451,0.2996)
(452,0.3039)
(453,0.3036)
(454,0.2959)
(455,0.3068)
(456,0.3012)
(457,0.2959)
(458,0.2924)
(459,0.3045)
(460,0.3042)
(461,0.3091)
(462,0.3032)
(463,0.2921)
(464,0.3091)
(465,0.3010)
(466,0.2970)
(467,0.3089)
(468,0.2991)
(469,0.2953)
(470,0.3037)
(471,0.2977)
(472,0.3086)
(473,0.2969)
(474,0.2974)
(475,0.3017)
(476,0.2951)
(477,0.3117)
(478,0.3057)
(479,0.3030)
(480,0.3005)
(481,0.3091)
(482,0.3087)
(483,0.3111)
(484,0.3016)
(485,0.3050)
(486,0.3013)
(487,0.3050)
(488,0.3097)
(489,0.3048)
(490,0.3079)
(491,0.2888)
(492,0.3008)
(493,0.3153)
(494,0.3154)
(495,0.2949)
(496,0.3102)
(497,0.3059)
(498,0.3061)
(499,0.3039)
};\addlegendentry{Q\gls{zo}}

\end{axis}
\end{tikzpicture}
        \caption{Train Reward}
        \label{fig:es_train_reward2}
    \end{subfigure}
    \hfill
    \begin{subfigure}{0.48\textwidth}
        \centering
\begin{tikzpicture}
\begin{axis}[
    width=\textwidth,
    height=6cm,
    xlabel={Iterations},
    ylabel={Test Reward},
    xmin=0,
    xmax=500,
    ymin=0,
    ymax=0.6,
    legend style={at={(0.97,0.03)}, anchor=south east, font=\small},
    ymajorgrids=true,
    xmajorgrids=true,
    grid style={gray!30},
    line width=1.2pt,
]

\addplot[color=teal, mark=none, mark size=2pt, each nth point=5,filter discard warning=false]coordinates{
(0,0.0751)
(24,0.1995)
(49,0.2774)
(74,0.3110)
(99,0.3523)
(124,0.3765)
(149,0.4444)
(174,0.4672)
(199,0.4471)
(224,0.4859)
(249,0.4769)
(274,0.4728)
(299,0.5102)
(324,0.5209)
(349,0.5003)
(374,0.5065)
(399,0.5306)
(424,0.5208)
(449,0.5550)
(474,0.5520)
(499,0.5565)
};

\addplot[color=red, mark=none, mark size=2pt]coordinates{
(0,0.0849)
(24,0.2171)
(49,0.3196)
(74,0.3693)
(99,0.3670)
(124,0.4690)
(149,0.4840)
(174,0.4925)
(199,0.5142)
(224,0.5137)
(249,0.5091)
(274,0.4952)
(299,0.5064)
(324,0.5669)
(349,0.5311)
(374,0.5460)
(399,0.5655)
(424,0.5803)
(449,0.5716)
(474,0.5752)
(499,0.5817)
};

\addplot[color=black, mark=none, mark size=2pt]coordinates{
(0,0.0701)
(24,0.1000)
(49,0.1178)
(74,0.1499)
(99,0.1494)
(124,0.1475)
(149,0.1672)
(174,0.1557)
(199,0.1768)
(224,0.1898)
(249,0.1894)
(274,0.2057)
(299,0.1921)
(324,0.1911)
(349,0.2263)
(374,0.1837)
(399,0.2149)
(424,0.2664)
(449,0.2172)
(474,0.2423)
(499,0.2242)
};

\end{axis}
\end{tikzpicture}
        \caption{Test Reward}
        \label{fig:es_test_reward2}
    \end{subfigure}
    \caption{Reward curves during fine-tuning on the Countdown task for full-parameter MeZO, Q\gls{zo}, and the proposed TerMeZO. Falcon-E-3B-Instruct~\citep{tiionebitllms} is used in this experiment.}
    \label{fig:es_countdown_2}
\end{figure}

\newpage

\section{Unbiasedness of SPSA for the Gaussian-Smoothed Objective}
\label{app:spsa}

Here, we prove that the symmetric two-point SPSA estimator with Gaussian perturbations is an unbiased estimator of the gradient of the Gaussian-smoothed objective $\tilde F_\epsilon(\w,\btheta)$ with respect to $\w$. Throughout this appendix $\btheta$ is fixed and omitted from the notation. Let $\z \sim \mathcal{N}(\mathbf{0}, \mathbf{I}_d)$ denote a standard Gaussian random vector in $\R^d$, with density
\begin{equation}
    \varphi(\z) = (2\pi)^{-d/2} \exp\!\left( -\tfrac{1}{2}\|\z\|^2 \right).
    \label{eq:gaussian_density}
\end{equation}
The Gaussian-smoothed objective at smoothing level $\epsilon > 0$ is
\begin{equation}
    \tilde F_\epsilon(\w) \;=\; \E_{\z}\!\left[ \tilde F(\w + \epsilon \z) \right] \;=\; \int_{\R^d} \tilde F(\w + \epsilon \z)\, \varphi(\z)\, d\z.
    \label{eq:smoothed_def}
\end{equation}
We show in Lemma~\ref{lem:smooth} (Appendix~\ref{app:proof_thm}) that although $\tilde F$ is piecewise constant, and hence discontinuous, in $\w$, the smoothed objective $\tilde F_\epsilon$ is infinitely differentiable in $\w$. 
Substituting $\bu = \w + \epsilon \z$ in~\eqref{eq:smoothed_def} yields
\begin{equation}
    \tilde F_\epsilon(\w) \;=\; \epsilon^{-d} \int_{\R^d} \tilde F(\bu)\, \varphi\!\left( \tfrac{\bu - \w}{\epsilon} \right) d\bu.
    \label{eq:smoothed_conv}
\end{equation}
Differentiating~\eqref{eq:smoothed_conv} with respect to $\w$ and using $\nabla_\w \varphi\!\left( \tfrac{\bu - \w}{\epsilon} \right) = \tfrac{1}{\epsilon} \tfrac{\bu-\w}{\epsilon}\, \varphi\!\left( \tfrac{\bu-\w}{\epsilon} \right)$, we obtain
\begin{equation}
    \nabla_\w \tilde F_\epsilon(\w) \;=\; \frac{1}{\epsilon}\, \E_{\z}\!\left[ \tilde F(\w + \epsilon \z)\, \z \right].
    \label{eq:smoothed_grad}
\end{equation}
The estimator in~\eqref{eq:smoothed_grad} is one-sided. To obtain the symmetric two-point form, observe that, since $\z$ and $-\z$ have the same distribution,
\begin{equation}
    \E_{\z}\!\left[ \tilde F(\w + \epsilon \z)\, \z \right] \;=\; -\,\E_{\z}\!\left[ \tilde F(\w - \epsilon \z)\, \z \right].
    \label{eq:symmetry}
\end{equation}
Combining~\eqref{eq:smoothed_grad} with~\eqref{eq:symmetry} yields
\begin{equation}
    \nabla_\w \tilde F_\epsilon(\w) \;=\; \E_{\z}\!\left[ \frac{\tilde F(\w + \epsilon \z) - \tilde F(\w - \epsilon \z)}{2\epsilon}\, \z \right],
    \label{eq:two_point_unbiased}
\end{equation}
which is the claimed identity.

\section{Proof of Theorem~\ref{thm:main}}
\label{app:proof_thm}

Before stating the assumptions, we use the following notations for the updates.
The two blocks $\w$ and $\btheta$ are updated simultaneously by
\begin{align}
    \w_{t+1} = \w_t - \eta\,\hat\bg_{\w,t},
    &\qquad
    \hat\bg_{\w,t}
    = \frac{\tilde f_t(\w_t+\epsilon\bar\z_t,\btheta_t) - \tilde f_t(\w_t-\epsilon\bar\z_t,\btheta_t)}{2\epsilon}\,\bar\z_t,
    \label{eq:ternary_update}\\
    \btheta_{t+1} = \btheta_t - \mu\,\hat\bg_{\btheta,t},
    &\qquad
    \hat\bg_{\btheta,t}
    = \frac{\tilde f_t(\w_t,\btheta_t+\epsilon\z_{\btheta,t}) - \tilde f_t(\w_t,\btheta_t-\epsilon\z_{\btheta,t})}{2\epsilon}\,\z_{\btheta,t},
    \label{eq:fp_update}
\end{align}
where $\tilde f_t(\w,\btheta)=f(Q(\w),\btheta;\mathcal{B}_t)$ is the minibatch loss.

\subsection{Assumptions}
\label{app:assumptions}

\begin{assumption}[Bounded loss]
\label{ass:bounded}
The loss $\tilde F$ is uniformly bounded: $|\tilde F(\w,\btheta)| = |F(Q(\w),\btheta)| \leq B$
for all $(\w,\btheta) \in \R^{d+p}$ and some constant $B > 0$.
\end{assumption}

\begin{assumption}[Lipschitz in $\w$, smooth in $\btheta$]
\label{ass:lipschitz}
$F$ is $L$-Lipschitz continuous in $\w$, and $\nabla_\btheta F$ is $L$-Lipschitz
in $(\w,\btheta)$. We use a single constant $L$ throughout, which is
without loss of generality (otherwise take the maximum). That
is, for all $\w,\w'\in\R^d$ and $\btheta,\btheta'\in\R^p$,
\begin{align}
    |F(\w,\btheta) - F(\w',\btheta)| &\leq L\,\|\w - \w'\|,\\
    \|\nabla_\btheta F(\w,\btheta) - \nabla_\btheta F(\w',\btheta')\| &\leq L\,\bigl(\|\w-\w'\| + \|\btheta - \btheta'\|\bigr). 
\end{align}
Each minibatch loss $f(\cdot,\cdot\,;\mathcal{B})$ satisfies the same two inequalities with a random constant $L(\mathcal{B})$ in place of $L$, where $\E_\mathcal{B}[L(\mathcal{B})^2]\leq L^2$.
\end{assumption}

\begin{assumption}[Mask alignment]
\label{ass:sparse}
The masked ternary gradient estimate \eqref{eq:ternary_update} is aligned with the smoothed gradient:
\begin{align}
    \bigl\langle \nabla_\w\tilde F_\epsilon(\w_t,\btheta_t),\;
    \E_t[\hat\bg_{\w,t}]\bigr\rangle
    \;\geq\; c\,\|\nabla_\w\tilde F_\epsilon(\w_t,\btheta_t)\|^2,
    \qquad c\in(0,1],
\end{align}
where $\E_t$ denotes expectation over the perturbations and the minibatch drawn at iteration $t$, with $\w_t$, $\btheta_t$ and $\m_t$ fixed.
\end{assumption}

\begin{assumption}[Minibatch variance]
\label{ass:variance}
For all $\tilde{\mathbf w}\in\{-1,0,1\}^d$ and $\btheta\in\R^p$,
\begin{equation}
    \E_{\mathcal{B}_t}\bigl\|\nabla_\btheta f(\tilde{\mathbf w},\btheta;\mathcal{B}_t)-\nabla_\btheta F(\tilde{\mathbf w},\btheta)\bigr\|^2\leq\sigma_B^2 .
\end{equation}
For a minibatch of $n_B$ i.i.d.\ samples, $\sigma_B^2=\sigma^2/n_B$, with $\sigma^2$ the per-sample variance.
\end{assumption}

\subsection{Proof of Theorem~\ref{thm:main}}
\label{app:sub_thm_proof}

\noindent\textbf{Notation.} 
For $S\subseteq[d]$, split the perturbation vector $\z$  as $\z=\bar\z+\z_{S^c}$, where $\bar\z$ keeps the coordinates of $\z$ in $S$ and $\z_{S^c}$ the remaining ones. Let
$D_S(\w,\z) = \|Q(\w+\epsilon\bar\z)-Q(\w)\|^2=\sum_{i\in S}(Q(w_i+\epsilon z_i)-Q(w_i))^2$
be the squared change of the quantized weights in $S$, and let
$q_S(\w) = \E[D_S(\w,\z)]/|S|$.
With this notation, $k_t=|S_t|$ is the mask size at iteration $t$, $q_t = q_{S_t}(\w_t)$ is the average expected squared change of the masked weights.

\begin{lemma}
\label{lem:smooth}
Under Assumptions~\ref{ass:bounded} and~\ref{ass:lipschitz}, the following hold.
\begin{itemize}[leftmargin=*,itemsep=0.1pt]
    \item $\tilde F_\epsilon$ is $C^\infty$ in $\w$ for any fixed $\btheta$, and continuously differentiable in $(\w,\btheta)$.
    \item Let $S\subseteq[d]$ with $|S|\leq k_0$. Let $\w,\bu\in\R^d$ with $\bu$ supported on $S$, and let $\btheta,\btheta'\in\R^p$. Then, we have
\end{itemize}
\begin{align}
    \text{(i)}\;\; & \bigl|\bu^\top\nabla^2_{\w\w}\tilde F_\epsilon(\w,\btheta)\,\bu\bigr|
   \leq L_\w\,\|\bu\|^2, \label{eq:lem_i}\\
    \text{(ii)}\;\; & \bigl|\langle\nabla_\w\tilde F_\epsilon(\w,\btheta)-\nabla_\w\tilde F_\epsilon(\w,\btheta'),\,\bu\rangle\bigr|
    \leq L_{\w\btheta}\,\|\bu\|\,\|\btheta-\btheta'\|, \label{eq:lem_ii}\\
    \text{(iii)}\;\; & \|\nabla_\btheta\tilde F_\epsilon(\w+\bu,\btheta)-\nabla_\btheta\tilde F_\epsilon(\w,\btheta)\|
    \leq L_{\w\btheta}\,\|\bu\|, \label{eq:lem_iii}\\
    \text{(iv)}\;\; & \|\nabla_\btheta\tilde F_\epsilon(\w,\btheta)-\nabla_\btheta\tilde F_\epsilon(\w,\btheta')\|
    \leq L\,\|\btheta-\btheta'\|, \label{eq:lem_iv}
\end{align}
with $L_\w = \sqrt{2k_0}\,L/\epsilon^2$ and $L_{\w\btheta} = \sqrt{k_0}\,L/\epsilon$.
\end{lemma}

\begin{proof}
\textit{Regularity of $\tilde F_\epsilon$.}
For fixed $\btheta$, $\tilde F_\epsilon(\w,\btheta) = (\tilde F(\cdot,\btheta) * \phi_\epsilon)(\w)$,
a convolution of the bounded measurable function $\tilde F(\cdot,\btheta)$ with the Gaussian kernel  $\phi_\epsilon(\bv) = \epsilon^{-d}(2\pi)^{-d/2}e^{-\|\bv\|^2/(2\epsilon^2)}$. $\phi_\epsilon$ belongs to the Schwartz class defined as
\begin{align}
    \mathcal{S}(\R^d) \;=\;
    \Bigl\{f \in C^\infty(\R^d) \;:\;
    \sup_{\x \in \R^d}
    \|\x\|^k |D^\alpha f(\x)| < \infty
    \;\;\forall\, k \in \mathbb{N},\; \forall\,\alpha \Bigr\}.
\end{align}
Using closure of $\mathcal{S}(\R^d)$ under differentiation and
its inclusion in $L^1(\R^d)$, together with Assumption~\ref{ass:bounded}
\begin{align}
    \int_{\R^d} |\tilde F(\bu,\btheta)|\,|D^\alpha\phi_\epsilon(\bu - \w)|\,d\bu
    \;\leq\; B \,\|D^\alpha\phi_\epsilon\|_{L^1(\R^d)} \;<\; \infty,
\end{align}
for every multi-index $\alpha$. By~\citep[Section~8.2]{folland1999real},
$\tilde F_\epsilon(\cdot,\btheta) \in C^\infty(\R^d)$, and derivatives
can be taken under the integral for Gaussian $\z$
\begin{align}
    \nabla_\w \tilde F_\epsilon(\w,\btheta)
    &\;=\; \frac{1}{\epsilon}\,\E_\z\bigl[\tilde F(\w+\epsilon\z,\btheta)\,\z\bigr],
    \label{eq:gradient_identity}\\
    \nabla^2_{\w\w}\tilde F_\epsilon(\w,\btheta)
    &\;=\; \frac{1}{\epsilon^2}\,\E_\z\bigl[\tilde F(\w+\epsilon\z,\btheta)\,(\z\z^\top - \mathbf{I}_d)\bigr].
    \label{eq:hessian_identity}
\end{align}
We now turn to the $\btheta$-gradient. Since $Q$ takes values in the finite set $\{-1,0,1\}^d$, Assumption~\ref{ass:lipschitz} implies that $\nabla_\btheta F(Q(\cdot),\btheta)$ is bounded, locally uniformly in $\btheta$. By dominated convergence,
\begin{align}
    \nabla_\btheta\tilde F_\epsilon(\w,\btheta) = \E_\z\bigl[\nabla_\btheta F(Q(\w+\epsilon\z),\btheta)\bigr].
    \label{eq:theta_gradient_identity}
\end{align}
Joint continuity of \eqref{eq:gradient_identity} and \eqref{eq:theta_gradient_identity} in $(\w,\btheta)$ follows in the same way, so $\tilde F_\epsilon\in C^1$.

Now define the objective smoothed only along $S$,
\begin{align}
    \tilde F^S_\epsilon(\w,\btheta) = \E_{\bar\z}\bigl[\tilde F(\w+\epsilon\bar\z,\btheta)\bigr].
\end{align}
Since $\bar\z$ and $\z_{S^c}$ are independent, the full smoothing is an average of partial smoothings at shifted points, that is,
\begin{align}
    \tilde F_\epsilon(\w,\btheta) = \E_{\z_{S^c}}\bigl[\tilde F^S_\epsilon(\w+\epsilon\z_{S^c},\btheta)\bigr].
    \label{eq:tower}
\end{align}
By the regularity argument above, applied to the objective smoothed at the masked coordinates in $S$, $\tilde F^S_\epsilon$ is $C^\infty$ along $S$, with derivatives bounded uniformly. Hence, for $\bu$ supported on $S$, we can write
\begin{align}
    \bu^\top\nabla^2_{\w\w}\tilde F_\epsilon(\w,\btheta)\,\bu
    =\frac{d^2}{ds^2}\E_{\z_{S^c}}\bigl[\tilde F^S_\epsilon(\w+s\bu+\epsilon\z_{S^c},\btheta)\bigr]_{s=0}
    =\E_{\z_{S^c}}\bigl[\bu^\top\nabla^2_{\w\w}\tilde F^S_\epsilon(\w+\epsilon\z_{S^c},\btheta)\,\bu\bigr],
    \label{eq:hessian_transfer}
\end{align}
and similarly for first derivatives along $\bu$. For such $\bu$, the partial smoothing has the same identities as \eqref{eq:gradient_identity}--\eqref{eq:hessian_identity} with $\z$ replaced by the masked $\bar\z$:
\begin{align}
    \langle\nabla_\w\tilde F^S_\epsilon(\w,\btheta),\bu\rangle&=\tfrac1\epsilon\E\bigl[\tilde F(\w+\epsilon\bar\z,\btheta)\,\bu^\top\bar\z\bigr],\\
    \bu^\top\nabla^2_{\w\w}\tilde F^S_\epsilon(\w,\btheta)\,\bu&=\tfrac{1}{\epsilon^2}\E\bigl[\tilde F(\w+\epsilon\bar\z,\btheta)\,((\bu^\top\bar\z)^2-\|\bu\|^2)\bigr].\label{hessian}
\end{align}
The shifted point $\w+\epsilon\z_{S^c}$ has the same coordinates in $S$ as $\w$, so $D_S$ and $q_S$ are the same at both points. Consequently, any bound on the derivatives of $\tilde F^S_\epsilon$ along $S$ that depends on $\w$ only through $D_S$ or $q_S(\w)$ transfers to $\tilde F_\epsilon$ with the same constant. It therefore suffices to prove (i)--(iii) for $\tilde F^S_\epsilon$. By the Lipschitz part of Assumption~\ref{ass:lipschitz} and the definition of $D_S$,
\begin{align}
    \bigl|\tilde F(\w+\epsilon\bar\z,\btheta) - \tilde F(\w,\btheta)\bigr|
    &\leq L\sqrt{D_S(\w,\z)},
    \label{eq:flip_F}\\
    \bigl\|\nabla_\btheta F(Q(\w+\epsilon\bar\z),\btheta) - \nabla_\btheta F(Q(\w),\btheta)\bigr\|
    &\leq L\sqrt{D_S(\w,\z)},
    \label{eq:flip_gradF}
\end{align}
where $\E[D_S(\w,\z)]=|S|\,q_S(\w)$ by definition. Squaring \eqref{eq:flip_F}--\eqref{eq:flip_gradF} and taking expectations over $\bar\z$ therefore gives
\begin{align}
    \E\bigl|\tilde F(\w+\epsilon\bar\z,\btheta)-\tilde F(\w,\btheta)\bigr|^2&\leq L^2|S|\,q_S(\w),
    \label{eq:flip_second_moment_F}\\
    \E\bigl\|\nabla_\btheta F(Q(\w+\epsilon\bar\z),\btheta)-\nabla_\btheta F(Q(\w),\btheta)\bigr\|^2&\leq L^2|S|\,q_S(\w).
    \label{eq:flip_second_moment}
\end{align}
Neglecting direct transition between $-1$ and $1$, whose probability is low for sufficiently small $\epsilon$, we have further
\begin{equation}
    q_S(\w)\leq 1.
\end{equation}

\medskip
\noindent\textit{Proof of (i).}
Let $\bu\in\R^d$ be supported on $S$, so that $\bu^\top\bar\z\sim\mathcal N(0,\|\bu\|^2)$. Since $\E[(\bu^\top\bar\z)^2-\|\bu\|^2]=0$, we may rewrite the Hessian identity \eqref{hessian} as
\begin{align}
    \bu^\top\nabla^2_{\w\w}\tilde F^S_\epsilon(\w,\btheta)\,\bu
    = \frac{1}{\epsilon^2}\,\E\Bigl[\bigl(\tilde F(\w+\epsilon\bar\z,\btheta)-\tilde F(\w,\btheta)\bigr)\bigl((\bu^\top\bar\z)^2-\|\bu\|^2\bigr)\Bigr].
\end{align}
We apply the Cauchy--Schwarz inequality. The first factor is bounded using \eqref{eq:flip_second_moment_F}. For the second, $(\bu^\top\bar\z)^2/\|\bu\|^2\sim\chi^2(1)$ has variance $2$, so $\E[((\bu^\top\bar\z)^2-\|\bu\|^2)^2]=2\|\bu\|^4$. Thus,
\begin{align}
    \bigl|\bu^\top\nabla^2_{\w\w}\tilde F^S_\epsilon(\w,\btheta)\,\bu\bigr|
    \leq \frac{1}{\epsilon^2}\cdot L\sqrt{|S|\,q_S(\w)}\cdot\sqrt2\,\|\bu\|^2
    = \frac{\sqrt{2|S|\,q_S(\w)}\,L}{\epsilon^2}\,\|\bu\|^2.
\end{align}
By \eqref{eq:tower}, the same bound holds for $\tilde F_\epsilon$. Since $q_S(\w)\leq1$ and $|S|\leq k_0$, we get \textit{(i)}.

\medskip
\noindent\textit{Proof of (ii).}
Let $\bu\in\R^d$ be supported on $S$, so that $\bu^\top\bar\z\sim\mathcal N(0,\|\bu\|^2)$. Let $g(\bar\z) = \tilde F(\w+\epsilon\bar\z,\btheta)-\tilde F(\w+\epsilon\bar\z,\btheta')$. Since $\E[\bu^\top\bar\z]=0$,
\begin{equation}
    \langle\nabla_\w\tilde F^S_\epsilon(\w,\btheta)-\nabla_\w\tilde F^S_\epsilon(\w,\btheta'),\bu\rangle
    = \tfrac1\epsilon\E\bigl[(g(\bar\z)-g(\mathbf 0))\,\bu^\top\bar\z\bigr].
\end{equation}
Write $\mathbf a=Q(\w+\epsilon\bar\z)$, $\mathbf b=Q(\w)$ and $\btheta_r=\btheta'+r(\btheta-\btheta')$. Then
\begin{equation}
    g(\bar\z)-g(\mathbf 0)=\int_0^1\langle\nabla_\btheta F(\mathbf a,\btheta_r)-\nabla_\btheta F(\mathbf b,\btheta_r),\btheta-\btheta'\rangle\,dr .
\end{equation}
First, for each $r\in[0,1]$, the Cauchy--Schwarz inequality for the inner product in $\R^p$ and the Lipschitz continuity of $\nabla_\btheta F$ in its first argument (Assumption~\ref{ass:lipschitz}) give
\begin{equation}
    \bigl|\langle\nabla_\btheta F(\mathbf a,\btheta_r)-\nabla_\btheta F(\mathbf b,\btheta_r),\btheta-\btheta'\rangle\bigr|
    \leq\|\nabla_\btheta F(\mathbf a,\btheta_r)-\nabla_\btheta F(\mathbf b,\btheta_r)\|\,\|\btheta-\btheta'\|
    \leq L\|\mathbf a-\mathbf b\|\,\|\btheta-\btheta'\|.
\end{equation}
The right-hand side does not depend on $r$, so integrating over $r\in[0,1]$ and using $\|\mathbf a-\mathbf b\|=\sqrt{D_S(\w,\z)}$ yields
\begin{equation}
    |g(\bar\z)-g(\mathbf 0)|\leq L\sqrt{D_S(\w,\z)}\,\|\btheta-\btheta'\|.
    \label{eq:g_bound}
\end{equation}
Applying the Cauchy--Schwarz inequality for expectations,
\begin{equation}
    \bigl|\E\bigl[(g(\bar\z)-g(\mathbf 0))\,\bu^\top\bar\z\bigr]\bigr|
    \leq L\sqrt{|S|\,q_S(\w)}\,\|\bu\|\,\|\btheta-\btheta'\|,
\end{equation}
and dividing by $\epsilon$, together with $q_S(\w)\leq1$ and $|S|\leq k_0$,
\begin{equation}
    |\langle\nabla_\w\tilde F^S_\epsilon(\w,\btheta)-\nabla_\w\tilde F^S_\epsilon(\w,\btheta'),\bu\rangle|
    \leq\tfrac{L}{\epsilon}\sqrt{|S|q_S(\w)}\,\|\bu\|\,\|\btheta-\btheta'\|\leq L_{\w\btheta}\|\bu\|\|\btheta-\btheta'\|,
\end{equation}
and by \eqref{eq:tower}, the same bound holds for $\tilde F_\epsilon$.

\medskip
\noindent\textit{Proof of (iii).} 
Fix a unit vector $\bv\in\R^p$. Let $\chi(\x)=\langle\bv,\nabla_\btheta F(Q(\x),\btheta)\rangle$, which is bounded and measurable, and let $\psi^S(\w)=\E[\chi(\w+\epsilon\bar\z)]$. By \eqref{eq:theta_gradient_identity} and \eqref{eq:tower}, $\langle\bv,\nabla_\btheta\tilde F_\epsilon(\w,\btheta)\rangle=\E_{\z_{S^c}}[\psi^S(\w+\epsilon\z_{S^c})]$.
Since $\E[\bu^\top\bar\z]=0$, \eqref{eq:flip_gradF} gives, for every $\w'$,
\begin{equation}
    |\langle\nabla\psi^S(\w'),\bu\rangle|
    =\tfrac1\epsilon\bigl|\E[(\chi(\w'+\epsilon\bar\z)-\chi(\w'))\,\bu^\top\bar\z]\bigr|
    \leq L_{\w\btheta}\|\bu\|.
\end{equation}
Let $\Psi(\w)=\langle\bv,\nabla_\btheta\tilde F_\epsilon(\w,\btheta)\rangle=\E_{\z_{S^c}}[\psi^S(\w+\epsilon\z_{S^c})]$. As for \eqref{eq:hessian_transfer}, the derivative along $\bu$ can be taken under the expectation, and every shifted point $\w'+\epsilon\z_{S^c}$ satisfies the bound above. Hence, for every $\w'\in\R^d$,
\begin{equation}
    |\langle\nabla\Psi(\w'),\bu\rangle|
    =\bigl|\E_{\z_{S^c}}\bigl[\langle\nabla\psi^S(\w'+\epsilon\z_{S^c}),\bu\rangle\bigr]\bigr|
    \leq L_{\w\btheta}\|\bu\|.
\end{equation}
The function $r\mapsto\Psi(\w+r\bu)$ is continuously differentiable on $[0,1]$ with derivative $\langle\nabla\Psi(\w+r\bu),\bu\rangle$. By the fundamental theorem of calculus,
\begin{equation}
    \bigl|\Psi(\w+\bu)-\Psi(\w)\bigr|
    =\Bigl|\int_0^1\langle\nabla\Psi(\w+r\bu),\bu\rangle\,dr\Bigr|
    \leq\int_0^1|\langle\nabla\Psi(\w+r\bu),\bu\rangle|\,dr
    \leq L_{\w\btheta}\|\bu\|,
\end{equation}
that is, $|\langle\bv,\nabla_\btheta\tilde F_\epsilon(\w+\bu,\btheta)-\nabla_\btheta\tilde F_\epsilon(\w,\btheta)\rangle|\leq L_{\w\btheta}\|\bu\|$ for every unit vector $\bv\in\R^p$. Taking the supremum over $\bv$ yields \eqref{eq:lem_iii}.

\medskip
\noindent\textit{Proof of (iv).}
By \eqref{eq:theta_gradient_identity}, applied at $\btheta$ and at $\btheta'$ with the same $\w$,
\begin{equation}
    \nabla_\btheta\tilde F_\epsilon(\w,\btheta)-\nabla_\btheta\tilde F_\epsilon(\w,\btheta')
    =\E_\z\bigl[\nabla_\btheta F(Q(\w+\epsilon\z),\btheta)-\nabla_\btheta F(Q(\w+\epsilon\z),\btheta')\bigr].
\end{equation}
Assumption~\ref{ass:lipschitz}, applied with $\w'=\w$, gives, for every $\w$,
\begin{equation}
    \bigl\|\nabla_\btheta F(\w,\btheta)-\nabla_\btheta F(\w,\btheta')\bigr\|\leq L\|\btheta-\btheta'\|.
\end{equation}
By Jensen's inequality, we get
\begin{equation}
    \bigl\|\nabla_\btheta\tilde F_\epsilon(\w,\btheta)-\nabla_\btheta\tilde F_\epsilon(\w,\btheta')\bigr\|
    \leq\E_\z\bigl\|\nabla_\btheta F(Q(\w+\epsilon\z),\btheta)-\nabla_\btheta F(Q(\w+\epsilon\z),\btheta')\bigr\|
    \leq L\|\btheta-\btheta'\|,
\end{equation}
which is \eqref{eq:lem_iv}.
\end{proof}

\medskip

\begin{lemma}
\label{lem:estimators}
Under Assumptions~\ref{ass:bounded}--\ref{ass:variance},  and with all expectations taken as $\E_t$ ( expectation over the perturbations and the minibatch drawn at iteration t), we have
\begin{align}
    \text{(a)}\;\; & \E\|\hat\bg_{\w,t}\|^2 \leq \frac{L^2}{2\epsilon^2}\,\sigma^2_{\w,t},
    \qquad \sigma^2_{\w,t}=k_t^2q_t+\sqrt3\,k_t\sqrt{q_t}; \label{eq:variance_w}\\
    \text{(b)}\;\; & \bigl\|\E\hat\bg_{\btheta,t} - \nabla_\btheta\tilde F_\epsilon(\w_t,\btheta_t)\bigr\|
    \leq b_t := L\bigl(\epsilon\sqrt p + \sqrt{N_t}\bigr); \label{eq:bias_theta}\\
    \text{(c)}\;\; & \E\|\hat\bg_{\btheta,t}\|^2 \leq
    4(p+2)\,\|\nabla_\btheta\tilde F_\epsilon(\w_t,\btheta_t)\|^2
    + 4(p+2)L^2 N_t
    \nonumber\\
    &\qquad+ 2(p+2)\,\sigma_B^2
    + \tfrac{\epsilon^2L^2}{2}\,p(p+2)(p+4). \label{eq:variance_theta}
\end{align}
\end{lemma}

\begin{proof}
\textit{(a)} Write $D^\pm = D_{S_t}(\w_t,\pm\z_t)$ for the squared changes under the two perturbations. A coordinate can change under at most one of the two perturbations under sufficiently small $\epsilon$. Hence, each coordinate contributes to at most one of $D^+,D^-$, that is,
\begin{equation}
    \|Q(\w_t+\epsilon\bar\z_t)-Q(\w_t-\epsilon\bar\z_t)\|^2=D^++D^- .
\end{equation}
Fix the minibatch $\mathcal{B}_t$. By Assumption~\ref{ass:lipschitz} for the minibatch loss, the finite difference $\Delta_t = \tilde f_t(\w_t+\epsilon\bar\z_t,\btheta_t)-\tilde f_t(\w_t-\epsilon\bar\z_t,\btheta_t)$, where $\tilde f_t(\w,\btheta)=f(Q(\w),\btheta;\mathcal{B}_t)$ is the quantized minibatch loss, therefore satisfies $\Delta_t^2\leq L(\mathcal{B}_t)^2(D^++D^-)$, and
\begin{equation}
    \|\hat\bg_{\w,t}\|^2=\frac{\Delta_t^2}{4\epsilon^2}\|\bar\z_t\|^2\leq\frac{L(\mathcal{B}_t)^2}{4\epsilon^2}(D^++D^-)\|\bar\z_t\|^2 .
\end{equation}
The map $\bar\z_t\mapsto-\bar\z_t$ preserves the law of $\bar\z_t$ so $\E[D^-\|\bar\z_t\|^2]=\E[D^+\|\bar\z_t\|^2]$. Write $D_i$ for the squared change of coordinate $i$; neglecting direct $\pm1$ transitions as before, $D_i\in\{0,1\}$. Let $\pi_i=\E D_i$, so that $\sum_{i\in S_t}\pi_i=k_tq_t$. Then
\begin{align}
    \E[D^+\|\bar\z_t\|^2]
    = \sum_{i,j\in S_t}\E[D_i\,z_j^2]
    \leq (k_t-1)\sum_{i\in S_t}\pi_i + \sum_{i\in S_t}\sqrt{3\pi_i}
    \leq k_t^2q_t+\sqrt3\,k_t\sqrt{q_t}.
\end{align}
The off-diagonal terms $i\neq j$ use independence of the coordinates. The diagonal terms use the Cauchy--Schwarz inequality with $\E z_i^4=3$ and $D_i^2=D_i$, followed by $\sum_i\sqrt{\pi_i}\leq\sqrt{k_t\sum_i\pi_i}$. Taking expectations over $\z_t$ in the bound on $\|\hat\bg_{\w,t}\|^2$ and using the symmetry gives, for a fixed minibatch,
\begin{equation}
    \E_{\z_t}\|\hat\bg_{\w,t}\|^2\leq\frac{L(\mathcal{B}_t)^2}{2\epsilon^2}\E[D^+\|\bar\z_t\|^2]\leq\frac{L(\mathcal{B}_t)^2}{2\epsilon^2}\,\sigma^2_{\w,t}.
    \label{eq:variance_w_batch}
\end{equation}
Taking the expectation over the minibatch and using $\E[L(\mathcal{B}_t)^2]\leq L^2$ proves (a).

\smallskip
\noindent\textit{(b)} Let $h_t(\btheta)=F(Q(\w_t),\btheta)$ and $h_{t,\mathcal{B}}(\btheta)=f(Q(\w_t),\btheta;\mathcal{B}_t)$, which are $L$-smooth and $L(\mathcal{B}_t)$-smooth by Assumption~\ref{ass:lipschitz}, and write $\bxi=\z_{\btheta,t}\in\R^p$. The estimator is linear in the loss and $\mathcal{B}_t$ is independent of $\bxi$, so averaging over the minibatch replaces $h_{t,\mathcal{B}}$ by $h_t$. By Stein's lemma, $\E[h_t(\btheta+\epsilon\bxi)\bxi]=\epsilon\,\E[\nabla h_t(\btheta+\epsilon\bxi)]$. Together with the symmetry $\bxi\mapsto-\bxi$, this gives
\begin{equation}
    \E\hat\bg_{\btheta,t}=\E_\bxi[\nabla h_t(\btheta_t+\epsilon\bxi)].
\end{equation}
We decompose the bias as
\begin{equation}
    \E\hat\bg_{\btheta,t}-\nabla_\btheta\tilde F_\epsilon(\w_t,\btheta_t)
    =\bigl[\E\hat\bg_{\btheta,t}-\nabla h_t(\btheta_t)\bigr]
    +\bigl[\nabla h_t(\btheta_t)-\nabla_\btheta\tilde F_\epsilon(\w_t,\btheta_t)\bigr],
    \label{eq:bias_decomposition}
\end{equation}
For the first term, since $\nabla h_t$ is $L$-Lipschitz,
\begin{equation}
    \|\E\hat\bg_{\btheta,t}-\nabla h_t(\btheta_t)\|
    =\bigl\|\E_\bxi[\nabla h_t(\btheta_t+\epsilon\bxi)-\nabla h_t(\btheta_t)]\bigr\|
    \leq L\epsilon\,\E\|\bxi\|\leq L\epsilon\sqrt p .
\end{equation}
For the second bracket, by \eqref{eq:theta_gradient_identity}, Assumption~\ref{ass:lipschitz} and Jensen's inequality,
\begin{align}
    \|\nabla h_t(\btheta_t)-\nabla_\btheta\tilde F_\epsilon(\w_t,\btheta_t)\|
    &\leq L\,\E_\z\|Q(\w_t)-Q(\w_t+\epsilon\z)\|
    \leq L\sqrt{\E D_{[d]}(\w_t,\z)}
    = L\sqrt{N_t}.
    \label{eq:bias_w_part}
\end{align}
Applying the triangle inequality to \eqref{eq:bias_decomposition} gives (b).

\smallskip
\noindent\textit{(c)} Fix the minibatch $\mathcal{B}_t$. By $L(\mathcal{B}_t)$-smoothness of $h_{t,\mathcal{B}}$,
\begin{equation}
    h_{t,\mathcal{B}}(\btheta_t+\epsilon\bxi)-h_{t,\mathcal{B}}(\btheta_t-\epsilon\bxi)=2\epsilon\langle\nabla h_{t,\mathcal{B}}(\btheta_t),\bxi\rangle+r,
    \qquad |r|\leq L(\mathcal{B}_t)\,\epsilon^2\|\bxi\|^2 .
\end{equation}
Hence
\begin{equation}
    \|\hat\bg_{\btheta,t}\|^2\leq 2\langle\nabla h_{t,\mathcal{B}},\bxi\rangle^2\|\bxi\|^2+\tfrac{L(\mathcal{B}_t)^2\epsilon^2}{2}\|\bxi\|^6 .
\end{equation}
For Gaussian $\bxi$, $\E[\langle\bg,\bxi\rangle^2\|\bxi\|^2]=(p+2)\|\bg\|^2$ and $\E\|\bxi\|^6=p(p+2)(p+4)$. Taking also the expectation over the minibatch, which is independent of $\bxi$, and using $\E[L(\mathcal{B}_t)^2]\leq L^2$ and $\E_\mathcal{B}\|\nabla h_{t,\mathcal{B}}\|^2=\|\nabla h_t\|^2+\E_\mathcal{B}\|\nabla h_{t,\mathcal{B}}-\nabla h_t\|^2\leq\|\nabla h_t\|^2+\sigma_B^2$ (Assumption~\ref{ass:variance}),
\begin{equation}
    \E\|\hat\bg_{\btheta,t}\|^2\leq 2(p+2)\|\nabla h_t(\btheta_t)\|^2+2(p+2)\sigma_B^2+\tfrac{\epsilon^2L^2}{2}p(p+2)(p+4).
\end{equation}
Finally, by \eqref{eq:bias_w_part},
\begin{equation}
    \|\nabla h_t(\btheta_t)\|^2\leq 2\|\nabla_\btheta\tilde F_\epsilon(\w_t,\btheta_t)\|^2+2L^2N_t ,
\end{equation}
which gives (c).
\end{proof}

\medskip
\noindent We now combine Lemmas~\ref{lem:smooth} and~\ref{lem:estimators} to prove Theorem~\ref{thm:main}, using the common Lipschitz constant $L$ of Assumption~\ref{ass:lipschitz} for both blocks. At iteration $t$, the two blocks take independent SPSA steps with Gaussian perturbations $\bar\z_t = \m_t\odot\z_t$ and $\z_{\btheta,t}\sim\mathcal N(\mathbf 0,\mathbf I_p)$, as in \eqref{eq:ternary_update}--\eqref{eq:fp_update}:
\begin{align}
    \w_{t+1} = \w_t - \eta_t\,\hat\bg_{\w,t},
    &\quad
    \hat\bg_{\w,t}
    = \frac{\tilde f_t(\w_t+\epsilon\bar\z_t,\btheta_t) - \tilde f_t(\w_t-\epsilon\bar\z_t,\btheta_t)}{2\epsilon}\,\bar\z_t,
    \label{eq:ternary_step}\\
    \btheta_{t+1} = \btheta_t - \mu_t\,\hat\bg_{\btheta,t},
    &\quad
    \hat\bg_{\btheta,t}
    = \frac{\tilde f_t(\w_t,\btheta_t+\epsilon\z_{\btheta,t}) - \tilde f_t(\w_t,\btheta_t-\epsilon\z_{\btheta,t})}{2\epsilon}\,\z_{\btheta,t},
    \label{eq:fp_step}
\end{align}
where $\tilde f_t(\w,\btheta)=f(Q(\w),\btheta;\mathcal{B}_t)$ is the quantized loss on the minibatch $\mathcal{B}_t$, and $\eta_t$ and $\mu_t$ are the learning rates of the latent weights and of the full-precision parameters. The step $\w_{t+1}-\w_t=-\eta_t\hat\bg_{\w,t}$ is supported on $S_t$, with $|S_t|=k_t\leq k_0$ since the mask can only shrink. Lemma~\ref{lem:smooth} therefore gives, with all gradients evaluated at $(\w_t,\btheta_t)$,
\begin{align}
    \tilde F_\epsilon(\w_{t+1},\btheta_{t+1})
    &\;\leq\;
    \tilde F_\epsilon(\w_t,\btheta_t)
    + \langle\nabla_\w\tilde F_\epsilon,\, \w_{t+1}-\w_t\rangle
    + \frac{L_\w}{2}\|\w_{t+1}-\w_t\|^2
    \nonumber\\
    &\quad+ \langle\nabla_\btheta\tilde F_\epsilon,\, \btheta_{t+1}-\btheta_t\rangle
    + \frac{L}{2}\|\btheta_{t+1}-\btheta_t\|^2
    \nonumber\\
    &\quad+ L_{\w\btheta}\,\|\w_{t+1}-\w_t\|\,\|\btheta_{t+1}-\btheta_t\|.
    \label{eq:smooth_descent}
\end{align}
In \eqref{eq:smooth_descent}, \eqref{eq:lem_i} and \eqref{eq:lem_iv} give the two quadratic terms, and \eqref{eq:lem_ii}--\eqref{eq:lem_iii} each contribute half of the last term, which accounts for the fact that both blocks move in the same step. \eqref{eq:lem_i}--\eqref{eq:lem_iii} apply because $\w_{t+1}-\w_t$ is supported on $S_t$.
Substituting $\w_{t+1}-\w_t = -\eta_t\hat\bg_{\w,t}$ and $\btheta_{t+1}-\btheta_t = -\mu_t\hat\bg_{\btheta,t}$ into \eqref{eq:smooth_descent},
\begin{align}
    \tilde F_\epsilon(\w_{t+1},\btheta_{t+1})
    \;\leq\;
    \tilde F_\epsilon(\w_t,\btheta_t)
    &- \eta_t \langle\nabla_\w\tilde F_\epsilon,\,\hat\bg_{\w,t}\rangle
    + \frac{L_\w \eta_t^2}{2}\|\hat\bg_{\w,t}\|^2
    \nonumber\\
    &- \mu_t\langle\nabla_\btheta\tilde F_\epsilon,\,\hat\bg_{\btheta,t}\rangle
    + \frac{L \mu_t^2}{2}\|\hat\bg_{\btheta,t}\|^2
    \nonumber\\
    &+ \eta_t\mu_t L_{\w\btheta}\,\|\hat\bg_{\w,t}\|\,\|\hat\bg_{\btheta,t}\|.
    \label{eq:descent}
\end{align}

\medskip
\noindent We take $\eta_t=\eta$ and $\mu_t=\mu$ satisfying
\begin{equation}
    \eta\leq\frac{c}{2L_\w},
    \qquad
    \mu\leq\min\Bigl\{\frac{1}{8L(p+2)},\;2c\eta\Bigr\}.
    \label{eq:step_conditions}
\end{equation}
The step sizes of Theorem~\ref{thm:main} satisfy \eqref{eq:step_conditions}, since $\mu=c^2/L_{\max}=\min\{c^2/L_\w,\,1/(8L(p+2))\}$ and $2c\eta=c^2/L_\w$. We take the expectation $\E_t$ in \eqref{eq:descent} and treat the terms separately. Assumption~\ref{ass:sparse} and \eqref{eq:variance_w} give
\begin{equation}
    -\eta\,\E_t\langle\nabla_\w\tilde F_\epsilon,\hat\bg_{\w,t}\rangle
    +\frac{L_\w\eta^2}{2}\E_t\|\hat\bg_{\w,t}\|^2
    \leq -a_\w\|\nabla_\w\tilde F_\epsilon\|^2+\frac{L_\w L^2\eta^2}{4\epsilon^2}\,\sigma^2_{\w,t},
    \qquad a_\w:=c\,\eta.
    \label{eq:inner_w}
\end{equation}
By \eqref{eq:bias_theta} and applying $xy\leq x^2/8+2y^2$, we can write
\begin{equation}
    \E_t\langle\nabla_\btheta\tilde F_\epsilon,\hat\bg_{\btheta,t}\rangle
    \geq\|\nabla_\btheta\tilde F_\epsilon\|^2-b_t\|\nabla_\btheta\tilde F_\epsilon\|
    \geq\tfrac78\|\nabla_\btheta\tilde F_\epsilon\|^2-2 b_t^2 ,
\end{equation}
with $b_t^2\leq 2L^2\epsilon^2p+2L^2N_t$. Since $\mu\leq1/(8L(p+2))$, we have $\tfrac{L\mu^2}{2}\leq\tfrac{\mu}{16(p+2)}$. Combining with \eqref{eq:variance_theta}, and keeping the minibatch term exact,
\begin{align}
    -\mu\,\E_t\langle\nabla_\btheta\tilde F_\epsilon,\hat\bg_{\btheta,t}\rangle
    + \frac{L\mu^2}{2}\,\E_t\|\hat\bg_{\btheta,t}\|^2
    &\leq -\frac{5\mu}{8}\|\nabla_\btheta\tilde F_\epsilon\|^2
    + \mu L^2\Bigl(4\epsilon^2p+\frac{17}{4}\,N_t+\frac{\epsilon^2p(p+4)}{32}\Bigr)
    \nonumber\\
    &\quad+L(p+2)\,\mu^2\sigma_B^2.
    \label{eq:inner_theta}
\end{align}
Conditionally on the minibatch, $\hat\bg_{\w,t}$ depends only on $\z_t$ and $\hat\bg_{\btheta,t}$ only on $\z_{\btheta,t}$, and these are independent. By Jensen's inequality and \eqref{eq:variance_w_batch}, $\E_{\z_t}\|\hat\bg_{\w,t}\|\leq L(\mathcal{B}_t)\sigma_{\w,t}/(\sqrt2\,\epsilon)$ with $\sigma_{\w,t}=(\sigma^2_{\w,t})^{1/2}$. The Cauchy--Schwarz inequality over the minibatch, with $\E[L(\mathcal{B}_t)^2]\leq L^2$, then gives
\begin{equation}
    \E_t\bigl[\|\hat\bg_{\w,t}\|\,\|\hat\bg_{\btheta,t}\|\bigr]
    \leq\frac{\sigma_{\w,t}}{\sqrt2\,\epsilon}\,\E_{\mathcal{B}_t}\Bigl[L(\mathcal{B}_t)\,\E_{\z_{\btheta,t}}\|\hat\bg_{\btheta,t}\|\Bigr]
    \leq\frac{L\sigma_{\w,t}}{\sqrt2\,\epsilon}\,\sqrt{\E_t\|\hat\bg_{\btheta,t}\|^2}.
\end{equation}
By \eqref{eq:variance_theta} and the subadditivity of the square root,
\begin{align}
    \sqrt{\E_t\|\hat\bg_{\btheta,t}\|^2}&\leq 2\sqrt{p+2}\,\|\nabla_\btheta\tilde F_\epsilon\|+R_t+\sqrt{2(p+2)}\,\sigma_B,
    \\
    R_t&=2L\sqrt{(p+2)\,N_t}+\epsilon L\sqrt{p(p+2)(p+4)/2}.
\end{align}
Let $A_t=\eta L_{\w\btheta}L\sigma_{\w,t}/(\sqrt2\,\epsilon)=\eta L_\w L\sigma_{\w,t}/2$, using $L_{\w\btheta}=\epsilon L_\w/\sqrt2$. We apply Young's inequality to the three resulting products: $xy\leq x^2/8+2y^2$ to the first, $ab\leq a^2+b^2/4$ to the second, and $ab\leq\lambda a^2+b^2/(4\lambda)$ with $\lambda=1/(\eta L_\w)$ to the third. This gives
\begin{align}
    \eta\mu L_{\w\btheta}\,\E_t\bigl[\|\hat\bg_{\w,t}\|\,\|\hat\bg_{\btheta,t}\|\bigr]
    &\leq \mu A_t\Bigl(2\sqrt{p+2}\,\|\nabla_\btheta\tilde F_\epsilon\|+R_t+\sqrt{2(p+2)}\,\sigma_B\Bigr)
    \nonumber\\
    &\leq \frac\mu8\|\nabla_\btheta\tilde F_\epsilon\|^2
    + 9\mu A_t^2(p+2)
    + \mu\,\frac{R_t^2}{4(p+2)}
    \nonumber\\
    &\quad+ \mu\,\eta L_\w\Bigl(\frac{L^2(p+2)}{4}\,\sigma^2_{\w,t}+\frac{\sigma_B^2}{2}\Bigr),
    \label{eq:cross}
\end{align}
with $R_t^2/(4(p+2))\leq 2L^2N_t+\epsilon^2L^2p(p+4)/4$ and $9A_t^2=\tfrac94\eta^2L_\w^2L^2\sigma^2_{\w,t}$.
Adding \eqref{eq:inner_theta} and \eqref{eq:cross}, and using $4\epsilon^2p+\tfrac{9}{32}\epsilon^2p(p+4)\leq\tfrac12\epsilon^2(p+6)^2$ and $\tfrac{25}{4}\leq8$, the full-precision and cross-block terms together contribute at most, with $a_\btheta:=\mu/2$,
\begin{align}
    -a_\btheta\|\nabla_\btheta\tilde F_\epsilon\|^2
    &+ \mu L^2\Bigl(\frac{\epsilon^2(p+6)^2}{2}+8\,N_t\Bigr)
    + \mu\Bigl(\frac94\eta^2L_\w^2+\frac{\eta L_\w}{4}\Bigr)L^2(p+2)\,\sigma^2_{\w,t}
    \nonumber\\
    &+ \mu\Bigl(L(p+2)\mu+\frac{\eta L_\w}{2}\Bigr)\sigma_B^2.
    \label{eq:theta_total}
\end{align}

\medskip
\noindent We insert \eqref{eq:inner_w} and \eqref{eq:theta_total} into \eqref{eq:descent} and take expectation. Summing over $t=0,\dots,T-1$ and dividing by $T$ gives the bound
\begin{align}
    \frac1T\sum_{t=0}^{T-1}\E\Bigl[a_\w\|\nabla_\w\tilde F_\epsilon\|^2+a_\btheta\|\nabla_\btheta\tilde F_\epsilon\|^2\Bigr]
    &\leq
    \frac{\Delta_0}{T}
    + \frac{L_\w L^2\eta^2}{4\epsilon^2}\,\overline{\sigma^2_\w}
    + \mu L^2\Bigl(\frac{\epsilon^2(p+6)^2}{2}+8\,\overline{N}\Bigr)
    \nonumber\\
    &\quad+ \mu\Bigl(\frac94\eta^2L_\w^2+\frac{\eta L_\w}{4}\Bigr)L^2(p+2)\,\overline{\sigma^2_\w}
    \nonumber\\
    &\quad+ \mu\Bigl(L(p+2)\mu+\frac{\eta L_\w}{2}\Bigr)\sigma_B^2.
    \label{eq:weighted_bound}
\end{align}

By \eqref{eq:step_conditions}, $\mu\leq2c\eta=2a_\w$, so the smaller weight is $a_\btheta=\mu/2$. Dividing \eqref{eq:weighted_bound} by $\mu/2$ gives, for any step sizes satisfying \eqref{eq:step_conditions},
\begin{align}
    \frac1T\sum_{t=0}^{T-1}\E\|\nabla\tilde F_\epsilon\|^2
    &\leq\frac{2\Delta_0}{\mu T}
    +L^2\epsilon^2(p+6)^2
    +16L^2\,\overline{N}
    +\frac{L_\w L^2\eta^2}{2\mu\epsilon^2}\,\overline{\sigma^2_\w}
    \nonumber\\
    &\quad+\Bigl(\frac92\eta^2L_\w^2+\frac{\eta L_\w}{2}\Bigr)L^2(p+2)\,\overline{\sigma^2_\w}
    +\bigl(2L(p+2)\mu+\eta L_\w\bigr)\sigma_B^2.
    \label{eq:general_bound}
\end{align}
With $\eta=c/(2L_\w)$ and $\mu=c^2/L_{\max}$, we have
\begin{equation}
    \frac{2}{\mu}=\frac{2L_{\max}}{c^2},
    \qquad
    \frac{L_\w\eta^2}{2\mu}=\frac{L_{\max}}{8L_\w},
    \qquad
    \frac92\eta^2L_\w^2+\frac{\eta L_\w}{2}=\frac98c^2+\frac c4\leq\frac{11}{8}c,
\end{equation}
and $2L(p+2)\mu+\eta L_\w=2c^2L(p+2)/L_{\max}+c/2\leq\tfrac14+\tfrac12=\tfrac34$, using $L_{\max}\geq8c^2L(p+2)$ and $c\leq1$. Substituting into \eqref{eq:general_bound}, the terms in $\overline{\sigma^2_\w}$ become $L^2\,\overline{\sigma^2_\w}\bigl(\frac{L_{\max}}{8L_\w\epsilon^2}+\frac{11}{8}c(p+2)\bigr)$. Since $L_\w\epsilon^2=\sqrt{2k_0}\,L$, and since $L_{\max}\geq8c^2L(p+2)$ gives $\frac{11}{8}cL(p+2)\leq\frac{11}{64}\frac{L_{\max}}{c}\leq\frac{L_{\max}}{c}$,
\begin{equation}
    L^2\,\overline{\sigma^2_\w}\Bigl(\frac{L_{\max}}{8L_\w\epsilon^2}+\frac{11}{8}c(p+2)\Bigr)
    \leq L\,L_{\max}\,\overline{\sigma^2_\w}\Bigl(\frac{1}{8\sqrt{2k_0}}+\frac1c\Bigr),
\end{equation}
which yields 
\begin{align}
    \frac1T\sum_{t=0}^{T-1}\E\|\nabla\tilde F_\epsilon(\w_t,\btheta_t)\|^2
    \;\leq\;
    &\frac{2L_{\max}\Delta_0}{c^2T}
    + L^2\epsilon^2(p+6)^2
    + 16L^2\,\overline{N}
    + \frac34\,\sigma_B^2
    + L\,L_{\max}\,\overline{\sigma^2_\w}\,\Bigl(\frac{1}{8\sqrt{2k_0}}+\frac{1}{c}\Bigr).
\end{align}
Noting that 
\begin{align}
C &=    L^2\epsilon^2(p+6)^2
    + 16L^2\,\overline{N}
    + \frac34\,\sigma_B^2
    + L\,L_{\max}\,\overline{\sigma^2_\w}\,\Bigl(\frac{1}{8\sqrt{2k_0}}+\frac{1}{c}\Bigr)\\&
    = \mathcal{O}\!\left( \epsilon^2 p^2 + \overline{N} + \sigma_B^2
      + L_{\max}\,\overline{\sigma^2_\w} \right),
\end{align}
completes the proof.
\hfill$\square$

\newpage

\section{Experiments Details}
\label{sec:exp_details}

\subsection{GLUE/SuperGLUE Experiments}
\label{app:glue_exp}
Unless stated otherwise, the fraction of trainable latent weights is set to $\rho_0 = 0.05$ for TerMeZO and sparse MeZO variants. For all \gls{zo} methods, $K = 5$ Gaussian perturbation directions per
\gls{zo} step are applied with a perturbation scale $\epsilon = 10^{-3}$, and the
learning rate follows a linear decay schedule.

Following Me\gls{zo} protocol~\citep{malladi2023fine}, fine-tuning is cast as language modeling, i.e., each example is
rendered with a task-specific prompt template. We use the full training split of each task and report the test accuracy. 
The ICL baseline uses the same templates with $32$ demonstrations sampled from the train
split.

All \gls{zo} baselines use the same batch size, learning rate grid, and step budget (Table~\ref{tab:glue_hp_zo}). LoRA-Me\gls{zo} attaches rank-$8$
adapters ($\alpha = 16$) to every ternary linear layer, freezes the base ternary weights,
and ternary-quantizes the adapters in the forward pass so that the
no-full-precision-multiplication property of the BitNet architecture is preserved. Q\gls{zo}
freezes the ternary latent weights entirely and perturbs the per-channel quantization
scales together with the remaining full-precision parameters. FIM-S-Me\gls{zo} uses the same latent weight fraction $\rho_0$ as TerMe\gls{zo}, with the sparsity mask fixed a priori from a Fisher information estimate obtained by performing a small number of BP steps on the WikiText dataset~\citep{merity2016pointer}.

\begin{table}[ht]
\centering
\caption{Hyperparameters for \gls{zo} methods and BP on the GLUE/SuperGLUE datasets.} 
\label{tab:glue_hp_zo}
\begin{tabular}{lcc}
    \toprule
                     & ZO Methods & BP  \\
    \midrule
    Batch size       & $\{16, 32 \}$   & $\{8,16 \}$   \\
    Learning rate    & $\{1e^{-6}, 5e^{-6}, 1e^{-5}\}$ & $\{1e^{-5}, 5e^{-5}, 1e^{-4}\}$ \\
    Learning rate scheduler& Linear & Cosine\\
    \bottomrule
\end{tabular}
\end{table}

\subsection{Instruction Following and Reasoning Experiments}
\label{app:reasoning_exp}

\noindent\textbf{Datasets and protocol.} For the supervised fine-tuning experiments, we use
GSM8K~\citep{cobbe2021gsm8k} and Magicoder-Evol-Instruct-110K~\citep{wei2023magicoder}. Both are formatted with a single-turn instruction/response template,
\begin{center}
\texttt{Instruction: \{instruction\}\textbackslash nResponse: \{response\}}
\end{center}
and the loss is computed on the response tokens only.
Sequences are truncated to $2048$ tokens. GSM8K ships no validation split, so its test split is used only for the final evaluation, and the validation set is carved from the train split. For MagiCoder, we train on a quarter of the train set and use $1\%$ of the samples as the validation split.



\noindent\textbf{Hyperparameters.} Table~\ref{tab:instruct_hp} reports the settings for both model sizes Falcon-E-1B and Falcon-E-3B. All \gls{zo} methods share the same batch size, learning rate grid, and step budget. 
\begin{table}[h]
\centering
\caption{Hyperparameters for the instruction-following and reasoning experiments.
Shared across all \gls{zo} runs: $n = 5$ Gaussian perturbations, $\epsilon = 10^{-3}$,
linear learning rate decay, and $\rho_0 = 0.05$.}
\label{tab:instruct_hp}
\begin{tabular}{lcc}
    \toprule
    &  ZO & BP \\
    \midrule
       Batch size    & 16 & 8\\
       Learning rate       & $\{1e^{-5}, 5e^{-5}, 1e^{-4}\}$ & $\{1e^{-5}, 5e^{-5}, 1e^{-4}\}$ \\
       Learning rate scheduler& Linear & Cosine\\
    \bottomrule
\end{tabular}
\end{table}

\newpage
\section{Additional Results}
\label{app:quzo}

\subsection{Validation loss curves for TerMeZO and FIM-S-MeZO}
\label{lrs_ablation}

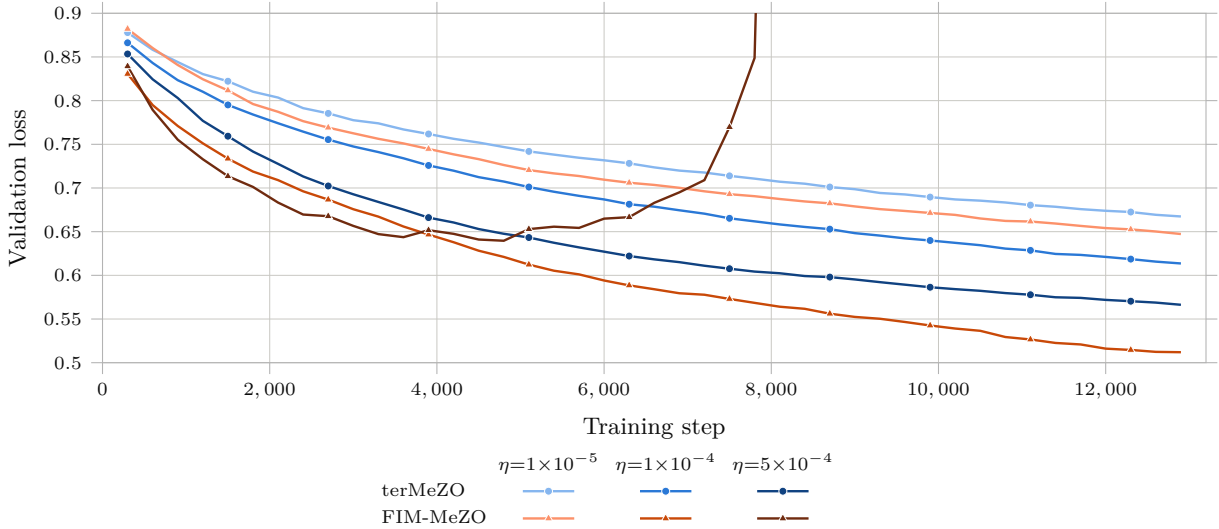
\begin{figure}[h]
  \definecolor{tera}{HTML}{86B6EF}
  \definecolor{terb}{HTML}{2A78D6}
  \definecolor{terc}{HTML}{104281}
  \definecolor{fima}{HTML}{FC916A}
  \definecolor{fimb}{HTML}{C94908}
  \definecolor{fimc}{HTML}{702B0F}
  \definecolor{labelink}{HTML}{0B0B0B}
  \definecolor{gridink}{HTML}{C9C8C3}
  \centering
  \begin{tikzpicture}
    \begin{axis}[
      name=main,
      width=0.98\linewidth, height=6.2cm,
      xlabel={Training step},
      ylabel={Validation loss},
      xmin=0, xmax=13200,
      ymin=0.50, ymax=0.90,
      xtick={0,2000,4000,6000,8000,10000,12000},
      ytick={0.50,0.55,0.60,0.65,0.70,0.75,0.80,0.85,0.90},
      tick align=outside, tick style={gray!60, thin},
      axis line style={gray!60},
      grid=major, grid style={gridink, very thin},
      label style={font=\small, text=labelink},
      tick label style={font=\scriptsize, text=labelink},
      scaled x ticks=false,
      xticklabel style={/pgf/number format/fixed, /pgf/number format/1000 sep={,}},
      yticklabel style={/pgf/number format/fixed, /pgf/number format/precision=2},
    ]
      \addplot[tera, line width=0.9pt, mark=*, mark size=1.5pt, mark repeat=4,
               mark options={fill=tera, draw=white, line width=0.35pt}] coordinates {
        (300,0.8779) (600,0.8582) (900,0.8442) (1200,0.8304) (1500,0.8222) (1800,0.8102) (2100,0.8035) (2400,0.7914) (2700,0.7854) (3000,0.7776) (3300,0.7740) (3600,0.7670) (3900,0.7618) (4200,0.7563) (4500,0.7520) (4800,0.7469) (5100,0.7419) (5400,0.7382) (5700,0.7346) (6000,0.7317) (6300,0.7282) (6600,0.7236) (6900,0.7199) (7200,0.7177) (7500,0.7139) (7800,0.7107) (8100,0.7072) (8400,0.7049) (8700,0.7011) (9000,0.6986) (9300,0.6943) (9600,0.6926) (9900,0.6896) (10200,0.6869) (10500,0.6855) (10800,0.6834) (11100,0.6804) (11400,0.6785) (11700,0.6758) (12000,0.6739) (12300,0.6725) (12600,0.6693) (12900,0.6674)
      };
      \addplot[terb, line width=0.9pt, mark=*, mark size=1.5pt, mark repeat=4,
               mark options={fill=terb, draw=white, line width=0.35pt}] coordinates {
        (300,0.8662) (600,0.8431) (900,0.8234) (1200,0.8102) (1500,0.7951) (1800,0.7841) (2100,0.7741) (2400,0.7645) (2700,0.7554) (3000,0.7475) (3300,0.7411) (3600,0.7340) (3900,0.7258) (4200,0.7198) (4500,0.7124) (4800,0.7074) (5100,0.7011) (5400,0.6957) (5700,0.6910) (6000,0.6869) (6300,0.6814) (6600,0.6786) (6900,0.6745) (7200,0.6706) (7500,0.6653) (7800,0.6618) (8100,0.6583) (8400,0.6554) (8700,0.6529) (9000,0.6483) (9300,0.6455) (9600,0.6423) (9900,0.6399) (10200,0.6372) (10500,0.6345) (10800,0.6306) (11100,0.6286) (11400,0.6245) (11700,0.6234) (12000,0.6211) (12300,0.6186) (12600,0.6157) (12900,0.6136)
      };
      \addplot[terc, line width=0.9pt, mark=*, mark size=1.5pt, mark repeat=4,
               mark options={fill=terc, draw=white, line width=0.35pt}] coordinates {
        (300,0.8534) (600,0.8244) (900,0.8030) (1200,0.7769) (1500,0.7593) (1800,0.7417) (2100,0.7276) (2400,0.7132) (2700,0.7023) (3000,0.6928) (3300,0.6840) (3600,0.6753) (3900,0.6661) (4200,0.6605) (4500,0.6530) (4800,0.6476) (5100,0.6433) (5400,0.6373) (5700,0.6320) (6000,0.6271) (6300,0.6221) (6600,0.6183) (6900,0.6151) (7200,0.6110) (7500,0.6076) (7800,0.6043) (8100,0.6024) (8400,0.5992) (8700,0.5980) (9000,0.5953) (9300,0.5922) (9600,0.5893) (9900,0.5864) (10200,0.5842) (10500,0.5823) (10800,0.5797) (11100,0.5778) (11400,0.5749) (11700,0.5742) (12000,0.5719) (12300,0.5704) (12600,0.5688) (12900,0.5663)
      };
      \addplot[fima, line width=0.9pt, mark=triangle*, mark size=1.8pt, mark repeat=4,
               mark options={fill=fima, draw=white, line width=0.35pt}] coordinates {
        (300,0.8821) (600,0.8602) (900,0.8408) (1200,0.8245) (1500,0.8117) (1800,0.7961) (2100,0.7872) (2400,0.7765) (2700,0.7691) (3000,0.7626) (3300,0.7563) (3600,0.7510) (3900,0.7447) (4200,0.7385) (4500,0.7330) (4800,0.7263) (5100,0.7205) (5400,0.7167) (5700,0.7137) (6000,0.7095) (6300,0.7060) (6600,0.7035) (6900,0.7002) (7200,0.6962) (7500,0.6929) (7800,0.6906) (8100,0.6874) (8400,0.6846) (8700,0.6824) (9000,0.6789) (9300,0.6758) (9600,0.6737) (9900,0.6714) (10200,0.6692) (10500,0.6651) (10800,0.6623) (11100,0.6616) (11400,0.6593) (11700,0.6566) (12000,0.6541) (12300,0.6526) (12600,0.6502) (12900,0.6473)
      };
      \addplot[fimb, line width=0.9pt, mark=triangle*, mark size=1.8pt, mark repeat=4,
               mark options={fill=fimb, draw=white, line width=0.35pt}] coordinates {
        (300,0.8306) (600,0.7949) (900,0.7710) (1200,0.7511) (1500,0.7336) (1800,0.7188) (2100,0.7090) (2400,0.6962) (2700,0.6867) (3000,0.6757) (3300,0.6672) (3600,0.6559) (3900,0.6467) (4200,0.6379) (4500,0.6282) (4800,0.6211) (5100,0.6123) (5400,0.6053) (5700,0.6011) (6000,0.5941) (6300,0.5885) (6600,0.5841) (6900,0.5795) (7200,0.5778) (7500,0.5730) (7800,0.5685) (8100,0.5641) (8400,0.5617) (8700,0.5561) (9000,0.5524) (9300,0.5503) (9600,0.5466) (9900,0.5426) (10200,0.5391) (10500,0.5365) (10800,0.5294) (11100,0.5266) (11400,0.5226) (11700,0.5209) (12000,0.5161) (12300,0.5146) (12600,0.5124) (12900,0.5120)
      };
      \addplot[fimc, line width=0.9pt, mark=triangle*, mark size=1.8pt, mark repeat=4,
               mark options={fill=fimc, draw=white, line width=0.35pt}] coordinates {
        (300,0.8393) (600,0.7894) (900,0.7553) (1200,0.7329) (1500,0.7135) (1800,0.7012) (2100,0.6833) (2400,0.6696) (2700,0.6677) (3000,0.6566) (3300,0.6472) (3600,0.6437) (3900,0.6517) (4200,0.6475) (4500,0.6410) (4800,0.6396) (5100,0.6529) (5400,0.6557) (5700,0.6543) (6000,0.6649) (6300,0.6666) (6600,0.6830) (6900,0.6950) (7200,0.7090) (7500,0.7697) (7800,0.8488) (8100,1.9458)
      };
      \node[anchor=east, font=\tiny, text=labelink, inner sep=2pt]
        at (axis cs:7900,0.875) {};
    \end{axis}

    \path (main.outer south west) -- (main.outer south east) coordinate[midway] (lgc);
    \begin{scope}[shift={(lgc)}, yshift=-0.30cm, xshift=0.25cm,
                  every node/.style={font=\scriptsize, text=labelink, inner sep=1pt}]
      \node[anchor=base] at (-1.05,0) {$\eta{=}1{\times}10^{-5}$};
      \node[anchor=base] at (0.50,0) {$\eta{=}1{\times}10^{-4}$};
      \node[anchor=base] at (2.05,0) {$\eta{=}5{\times}10^{-4}$};
      \node[anchor=base west] at (-3.30,-0.34) {terMeZO};
      \draw[tera, line width=0.9pt] (-1.39,-0.28) -- (-0.71,-0.28);
      \draw[tera] plot[mark=*, mark size=1.5pt, mark options={fill=tera, draw=white, line width=0.35pt}] coordinates {(-1.05,-0.28)};
      \draw[terb, line width=0.9pt] (0.16,-0.28) -- (0.84,-0.28);
      \draw[terb] plot[mark=*, mark size=1.5pt, mark options={fill=terb, draw=white, line width=0.35pt}] coordinates {(0.50,-0.28)};
      \draw[terc, line width=0.9pt] (1.71,-0.28) -- (2.39,-0.28);
      \draw[terc] plot[mark=*, mark size=1.5pt, mark options={fill=terc, draw=white, line width=0.35pt}] coordinates {(2.05,-0.28)};
      \node[anchor=base west] at (-3.30,-0.70) {FIM-MeZO};
      \draw[fima, line width=0.9pt] (-1.39,-0.64) -- (-0.71,-0.64);
      \draw[fima] plot[mark=triangle*, mark size=1.8pt, mark options={fill=fima, draw=white, line width=0.35pt}] coordinates {(-1.05,-0.64)};
      \draw[fimb, line width=0.9pt] (0.16,-0.64) -- (0.84,-0.64);
      \draw[fimb] plot[mark=triangle*, mark size=1.8pt, mark options={fill=fimb, draw=white, line width=0.35pt}] coordinates {(0.50,-0.64)};
      \draw[fimc, line width=0.9pt] (1.71,-0.64) -- (2.39,-0.64);
      \draw[fimc] plot[mark=triangle*, mark size=1.8pt, mark options={fill=fimc, draw=white, line width=0.35pt}] coordinates {(2.05,-0.64)};
    \end{scope}
  \end{tikzpicture}
  \caption{%
    Validation loss of TerMeZO and FIM-MeZO fine-tuning Falcon-E-3B-Base model on GSM8K at sparsity budget $\rho=0.01$, at different learning rates $\eta$. TerMeZO improves monotonically with $\eta$ and is stable throughout. FIM-MeZO has slightly better loss at different learning rates but it starts to diverge earlier than TerMeZO at higher learning rates.}
  \label{fig:val_loss_gsm8k_falcon_e_3b_rho0.01}
\end{figure}

Figure~\ref{fig:val_loss_gsm8k_falcon_e_3b_rho0.01} shows the validation loss of TerMeZO and FIM-S-MeZO when fine-tuning {Falcon-E-3B-Base} on GSM8K with sparsity budget $\rho=0.01$, for $\eta \in \{10^{-5}, 10^{-4}, 5\times10^{-4}\}$. TerMeZO decreases steadily at all three learning rates, and larger learning rates consistently yield lower loss. FIM-S-MeZO reaches a lower loss at $\eta=10^{-5}$ and $\eta=10^{-4}$, but becomes unstable at $\eta=5\times10^{-4}$. This behavior is consistent with our hypothesis that many of the weights selected by FIM-S-MeZO lie far from the quantization boundaries, so that perturbing them contributes mostly noise to the update.

\subsection{Comparison with QuZO}
\label{sec:quzo_comp}
We include additional experiments comparing TerMeZO with QuZO \citep{zhou2025quzo}, which perturbs the quantized weights directly via stochastic rounding. Table~\ref{tab:termezo_vs_quzo_gsm8k_1b} reports the GSM8K performance of both methods on {Falcon-E-1B-Base}. TerMeZO clearly outperforms QuZO, which barely improves over the base model. QuZO is also considerably more sensitive to the learning rate: it diverges at the learning rates used for TerMeZO and the other ZO baselines, which required us to search a lower range, and its best result is obtained with a learning rate more than two orders of magnitude smaller than that of TerMeZO.

\begin{table}[h]
  \centering
  \caption{%
    {Accuracy after finetuning of TerMeZO and QuZO on GSM8K for \texttt{Falcon-E-1B-Base}.} For each method, we report the best result over its learning-rate grid. We found QuZO diverges at the higher learning rates used for TerMeZO and the other ZO baselines and we searched a lower range for it to ensure convergence: $\eta \in \{ 10^{-6}, 3\times10^{-7}, 10^{-7}\}$.
  }
  \label{tab:termezo_vs_quzo_gsm8k_1b}
  \small
  \begin{tabular}{lcc}
    \toprule
    Method &  Test loss & GSM8K (\%) \\
    \midrule
    Base Model &     1.13    &    14.0\\
    \midrule
    TerMeZO  & $\mathbf{0.56}$ & $\mathbf{31.84}$ \\
    QuZO     & $0.99$ & $14.71$ \\
    \bottomrule
  \end{tabular}
\end{table}

\end{appendices}

\end{document}